\documentclass[journal,12pt,onecolumn,draftclsnofoot]{IEEEtran}

\usepackage{microtype}
\usepackage{graphicx}
\usepackage{subcaption}
\usepackage{booktabs} 

\usepackage{hyperref}

\usepackage{algorithm,algorithmic}
\usepackage{booktabs,tabularx,ragged2e,threeparttable}
\newcolumntype{Y}{>{\RaggedRight\arraybackslash}X}

\usepackage{amsmath}
\usepackage{amssymb}
\usepackage{mathtools}
\usepackage{amsthm}
\usepackage{enumitem}
\usepackage{array}
\usepackage{bbm}
\usepackage{multirow}
\usepackage[capitalize,noabbrev]{cleveref}

\theoremstyle{plain}
\newtheorem{theorem}{Theorem}[section]
\newtheorem{proposition}[theorem]{Proposition}
\newtheorem{lemma}[theorem]{Lemma}

\theoremstyle{definition}
\newtheorem{definition}[theorem]{Definition}
\newtheorem{assumption}[theorem]{Assumption}
\theoremstyle{remark}
\newtheorem{remark}[theorem]{Remark}

\newtheorem{fact}{Fact}

\newcommand{\aug}{\mathrm{aug}}

\newcommand{\cA}{\mathcal A}

\newcommand{\cC}{\mathcal C}

\newcommand{\cE}{\mathcal E}
\newcommand{\cF}{\mathcal F}

\newcommand{\cL}{\mathcal L}

\newcommand{\cN}{\mathcal N}
\newcommand{\cO}{\mathcal O}
\newcommand{\cP}{\mathcal P}
\newcommand{\cQ}{\mathcal Q}

\newcommand{\cS}{\mathcal S}

\newcommand{\cU}{\mathcal U}

\newcommand{\R}{\mathbb{R}}

\newcommand{\E}{\mathbb{E}}

\usepackage[textsize=tiny]{todonotes}

\usepackage[utf8]{inputenc} 
\usepackage[T1]{fontenc}    
\usepackage{hyperref}       
\usepackage{url}            
\usepackage{booktabs}       
\usepackage{amsfonts}       
\usepackage{nicefrac}       
\usepackage{microtype}      
\usepackage{xcolor}  
\def\BibTeX{{\rm B\kern-.05em{\sc i\kern-.025em b}\kern-.08em
    T\kern-.1667em\lower.7ex\hbox{E}\kern-.125emX}}
\begin{document}
\title{Provably Efficient Sample Complexity for Robust Constrained MDP}
\author{Sourav Ganguly, \IEEEmembership{Student-member, IEEE} and Arnob Ghosh, \IEEEmembership{Member, IEEE}
\thanks{S.Ganguly (email: sg2786@njit.edu)  and Dr. Arnob Ghosh (email: arnob.ghosh@njit.edu) are with the ECE Department at the New Jersey Institute of Technology (NJIT), Newark, New Jersey, USA. }}

\maketitle

\begin{abstract}

We study robust constrained Markov decision processes (RCMDPs), where the agent maximizes reward while ensuring that the total utility incurred exceeds a threshold under worst-case dynamics chosen adversarially from an uncertainty set. This models policies trained on a simulator that must remain safe when the real environment differs. Recent works give iteration complexity guarantees for RCMDPs via policy optimization, but sample complexity remains largely unexplored. We first show that, unlike in {\em unconstrained} robust MDPs, Markovian policies can be sub-optimal even under rectangular uncertainty sets. We then augment the state with the remaining utility budget and redefine the uncertainty set as a dynamic adversary on the augmented space. We propose Robust Constrained Value Iteration (RCVI), which solves a single-constraint linear program at each augmented state at each step. Using a generative model, RCVI needs $\mathcal{\tilde{O}}(|S||A|H^5/\epsilon^2)$ samples to violate the constraint by at most $\epsilon$ while attaining, up to $\epsilon$, the reward of every policy (including history-dependent ones) that meets the constraint along all trajectories; here $|S|$, $|A|$ are the state and action space sizes and $H$ is the horizon. To our knowledge, this is the {\em first sample complexity guarantee} for RCMDPs without a policy optimization oracle. Experiments validate our approach.
\end{abstract}
\vspace{-0.2in}
\section{Introduction}
Constrained Markov Decision Processes (CMDPs) provide a principled framework for
handling feasibility concerns in sequential decision-making: the agent seeks to
maximize expected reward while ensuring that the expected constraint cost (or
utility) remains within a predefined safety boundary~\cite{altman1998constrained}
(cf.~(\ref{eq:cmdp})). CMDPs have therefore been widely applied to keep agents
within safety limits~\cite{qiu2020upper,padakandla2022data}. In practice,
however, policies are trained on simulators or nominal models that do not
perfectly match the real environment. As a result, policies that satisfy CMDP
constraints in simulation may violate them when deployed.

To address this issue, recent works~\cite{Epigraph,ganguly_neurips} have studied
robust CMDPs (RCMDPs), where the goal is to maximize the worst-case reward while
ensuring that the worst-case utility remains above a threshold
(cf.~(\ref{eq:cmdp_robust})). RCMDPs are significantly more challenging than
standard CMDPs because strong duality fails~\cite{ma2025rectified,wang2022robust}.
This rules out the classical primal--dual approaches, which achieve
$\cO(1/\epsilon^2)$ sample complexity for CMDPs with a generative
model~\cite{vaswani2022near}.

Existing results~\cite{Epigraph,ganguly_neurips} establish an {\em iteration
complexity}\footnote{The number of iterations required to obtain a policy with
at most $\epsilon$ sub-optimality and at most $\epsilon$ violation.} of
$\cO(1/\epsilon^4)$, but they do not bound the sample complexity. Their
approaches evaluate the worst-case value function at every policy update, which
implicitly requires an oracle that samples from the nominal model, and the
number of samples this requires is not specified. A recent
work~\cite{ghosh2024sample} achieves $\tilde{\cO}(1/\epsilon^2)$ sample
complexity, but relies on a policy optimization oracle, which is generally
impractical. {\em More importantly, all these works optimize over policies that
are Markovian in the original state, which we show can be sub-optimal for
RCMDPs.} We consider the following question:
\begin{center}
{\em Can we achieve $\tilde{\cO}(1/\epsilon^2)$ sample complexity for RCMDPs
using a generative model without relying on a policy optimization oracle?}
\end{center}

We address this question for the episodic robust CMDP problem
\begin{align}\label{eq:cmdp_robust}
\max_{\pi}\ \min_{P\in\cP}V_{r,1}^{\pi,P}(s_1)\quad\text{subject to}\quad
\min_{P\in\cP}V_{g,1}^{\pi,P}(s_1)\ \ge\ b,
\end{align}
where $V_{r,1}^{\pi,P}(s_1)$ and $V_{g,1}^{\pi,P}(s_1)$ denote the expected total
reward and utility over the episode, starting from state $s_1$ under the
transition model $P$, and $\cP$ is the uncertainty set (see~(\ref{eq:uncertainty})).
The minimization is taken against a {\em dynamic} adversary that may choose
the transition kernel at every step as a function of the history. Robustness of
the constraint is therefore defined on the total utility incurred over the
episode (Section~\ref{sec:thm1}). For policies that are Markovian in the original
state, this coincides with the usual static adversary; for the history-dependent
policies that RCMDPs require, it need not (Appendix~\ref{app:static_dynamic}).

\begin{definition}\label{def:sample_complexity}
We seek a policy $\hat{\pi}$ such that, after $N_{\text{tot}}$ samples from the
{\em nominal model}, with high probability,
\begin{align}
& \mathrm{Sub\text{-}Opt}(\hat{\pi}) := V^{\diamond}_{r,1}(s_1)-\min_{P}V^{\hat{\pi},P}_{r,1}(s_1) \leq \epsilon,\nonumber\\
& \mathrm{Violation}(\hat{\pi}) := b-\min_P V_{g,1}^{\hat{\pi},P}(s_1) \leq \epsilon,
\end{align}
where $V^{\diamond}_{r,1}(s_1)$ is the largest worst-case reward attainable by a
policy, possibly history-dependent, that meets the threshold along every
trajectory the adversary can induce (Eq.~(\ref{eq:step_solve})).
\end{definition}
In contrast to unconstrained settings, both sub-optimality and violation must be
controlled. Note that we did not compare with the optimal solution of the problem (\ref{eq:cmdp_robust}), as it it is computationally difficult to track such an optimal solution especially as the adversary is now dynamic and can be history dependent. 

\textbf{Our Contributions}:
\begin{itemize}[leftmargin=*]
\item We show that Markovian policies in the original state space can be
strictly sub-optimal for RCMDPs even under rectangular uncertainty sets
(Lemma~\ref{lem:sub_optimality}), unlike in unconstrained robust MDPs and in
non-robust CMDPs. We also show that, unlike in robust MDPs, static and dynamic
adversaries can evaluate the same policy differently. The adversary model is
therefore part of the RCMDP problem definition.
\item Since computing the optimal policy among all the history-dependent policies can be computationally challenging, we thus compare with the best value function with the per-step feasibility. We augment the state with the remaining utility budget and show that
the uncertainty set must be redefined on the augmented space as a dynamic
adversary, which preserves rectangularity.   The
resulting Markov policy on the augmented state is feasible for the RCMDP and
earns at least the reward of every per-step feasible policy
(Theorem~\ref{thm:markov-suff-c}).
\item We propose Robust Constrained Value Iteration (RCVI), the sample-based
counterpart of this recursion. We prove that it requires
$\tilde{\cO}(|S||A|H^5/\epsilon^2)$ samples from a generative model to achieve
at most $\epsilon$ violation and at most $\epsilon$ sub-optimality in the sense
of Definition~\ref{def:sample_complexity}. Here $|S|$ and $|A|$ are the state
and action cardinalities and $H$ is the horizon. The bound holds for TV-distance,
$\chi^2$-distance, and KL-divergence uncertainty sets. To our knowledge, this is
the first sample-complexity bound for RCMDPs without oracle access.
\item Extensive experiments demonstrate the practical effectiveness of our
approach compared to existing approaches for RCMDPs.
\end{itemize}

\subsection{Other Related Works}
\textbf{CMDP:} The convex nature of the state-action occupancy measure ensures the existence of a zero duality gap between the primal and dual problem for CMDP, making them well-suited for solution via primal-dual methods~\cite{altman1998constrained,paternain2022safe,stooke2020responsive,liang2018accelerated,tessler2018reward,yu2019convergent,zheng2020constrained,efroni2020exploration,auer2008near}. The convergence bounds and rates of convergence for these methods have been extensively studied in~\cite{ding2020natural,li2024faster,liu2021policy,ying2022dual,wei2022triple,ghosh2022provably}.
Beyond primal-dual methods, LP-based and model-based approaches have been explored to solve the primal problem directly~\cite{achiam2017constrained,efroni2020exploration, chow2018lyapunov, dalal2018safe, xu2021crpo,yang2020projection}. 
However, the above approaches cannot be extended to the RCMDP case. 

\textbf{Robust MDP:} For robust (unconstrained) MDPs (introduced in \cite{iyengar2005robust}),  recent studies obtain the sample complexity guarantee using robust dynamic programming approach~\cite{panaganti2022sample,yang2022toward,shi2023curious,clavier2023towards,zhou2021finite}.  Model-free approaches are also studied ~\cite{shi2023curious, wang2023finite, wang2023policy,wang2021online, wang2023model,wang2023achieving, liang2023single, liu2022distributionally}. However, extending these methods to Robust Constrained MDPs (RCMDPs) presents additional challenges. The introduction of constraint functions complicates the optimization process as one needs to consider the worst value function both for the objective and the constraint. 

\textbf{RCMDP:} Unlike non-robust CMDPs, there is limited research available on robust environments. In \cite{wang2022robust,ma2025rectified}, it was shown that the optimization function for RCMDPs is not convex, making it difficult to solve the Lagrangian formulation, unlike in standard CMDPs. Some studies have attempted to address this challenge using a primal-dual approach \cite{mankowitz2020robust, wang2022robust} without any iteration complexity guarantee. 
\cite{zhangdistributionally} proposed a primal-dual approach to solve RCMDP under the strong duality by restricting to the categorical randomized policy class. However, they did not provide any iteration complexity guarantee.
As we discussed, \cite{Epigraph,ganguly_neurips,ma2025rectified} only consider iteration complexity and does not provide sample complexity guarantee. Moreover, all the above works consider Markovian policies only. 
\section{Problem Formulation}
\textbf{Constrained Markov Decision Problem}: A  constrained Markov Decision Process (CMDP) is characterized by the tuple $\{S,A,R,G,P,H\}$ where $S$ is the state-space, $A$ is the action-space; $R=\{r_h(s,a)\}$ and $G=\{g_h(s,a)\}$ are respectively the collection of rewards and utility for state-action pair $(s,a)$ at step $h\in [H]$. $H$ is the number of steps in an episode. $P_h$ denotes the transition probability $P_{h,s,a}(s^{\prime})=P_h(s^{\prime}|s,a)$ at step $h$. Without loss of generality, we assume that $r$, and $g$ are {\em deterministic}, and $|r(x,a)|\leq 1$, and $|g(x,a)|\leq 1$.  We denote $V_{r,t}^{\pi,P}(x)=\mathbb{E}_{\pi,P}[\sum_{h=t}^{H}r_h(x_h,a_h)|x_t=x]$ and $V_{g,t}^{\pi,P}(x)=\mathbb{E}_{\pi,P}[\sum_{h=t}^Hg_h(x_h,a_h)|x_t=x]$. We also denote $V_{j,t}^{\pi,P}(x)=\mathbb{E}_{\pi,P}[\sum_{h=t}^{H}j_h(x_h,a_h)|x_t=x]$ for $j=r,g$. In a CMDP \cite{efroni2020exploration,ghosh2022provably,ding2021provably,wei2022triple} setup one seeks to solve the following optimization problem. 
\begin{align}\label{eq:cmdp}
\max_{\pi} V_{r,1}^{\pi,P}(x)\quad \text{subject to }V_{g,1}^{\pi,P}(x)\geq b
\end{align}

\textbf{Robust CMDP}: We often use a simulator (or, nominal model)  to train our policy  before implementing in the real-life. However, the simulator setup and the real-life environment are often different, hence, we need a robust policy so that the policy can perform reasonably well in the real-life setup. Hence, even if we have a feasible policy $\pi$ for CMDP for a given nominal model $P$, it might be infeasible in the real environment as the true transition probability might be different.  In particular, we seek to solve the robust CMDP problem described in (\ref{eq:cmdp_robust}). 

In (\ref{eq:cmdp_robust}), $\mathcal{P}$ denotes the set of all transition probabilities. In particular, different transition probability defines different set of randomness inherent in the true environment.  The problem in (\ref{eq:cmdp_robust}) defines that we seek to maximize the worst case expected cumulative reward subject to the constraint that the worst case cumulative utility is above the threshold $b$.  Such robustness guarantee is important for implementing RL algorithms in practice.  The solution in (\ref{eq:cmdp_robust}) ensures that  the policy will still be safe even if there is a mismatch unlike the CMDP. 

Note that our analysis and approach can be easily applicable to the setting where $\max_P V_g^{\pi,P}\leq b$ as well where $g$ denotes the cost instead of utility at time-step $h$, and we are interested in the constraint such that the worst-case cost is below a certain threshold $b$. Our approach can be readily extended to the scenario where $r$ and $g$ are stochastic, and the distribution of $g$ is known. 
{\em For notational simplicity, we interchangably denote ${V}^{\pi}_{j}(x)=\min_{P\in \mathcal{P}}V^{\pi,P}_{j}(x)$ for $j=r,g$, and all $h$.} Note that the worst case model $P$ indeed depends on the policy which brings additional challenge.

\textbf{Uncertainty Set on models}: Similar to the one considered in the unconstrained episodic MDP setup \cite{xu2023improved}, we consider a set of transition probability models within a ball centered around the nominal model $P^0_{h,s,a}$ $\forall (h,s,a)\in  [H]\times S\times A$. We consider the uncertainty set $\mathcal{P}=\bigotimes_{(h,s,a)\in [H]\times  S\times A}\mathcal{P}_{h,s,a}$ such that 
\begin{align}\label{eq:uncertainty}
\mathcal{P}_{h,s,a}=\{P\in \Delta(S): D(P,P^0_{h,s,a})\leq \rho\}
\end{align}
where $D$ is the distance metric between two probability measures, and $\rho$ is the radius of the uncertainty set. This uncertainty set satisfies the $(s,a)$-rectangularity assumption \cite{iyengar2005robust,panaganti2022sample}. Our analysis can be extended trivially to $s$-rectangularity assumption as well \cite{yang2019sample}. Without rectangualarity assumption, even for unconstrained robust MDP, obtaining optimal policy is NP-hard problem \cite{wiesemann2013robust}. {\em We do not assume that that we know the nominal model $P^0$.} Rather, we can only access the nominal model. We consider the following distance metrics:
\begin{enumerate}[leftmargin=*]
\item \textbf{Total Variation uncertainty set}: Let $\mathcal{P}^{TV}=\bigotimes_{(h,s,a)\in [H]\times S\times A}\mathcal{P}^{TV}_{h,s,a}$ be the uncertainty set defined in (\ref{eq:uncertainty}) with total variation distance \cite{xu2023improved}
$D_{TV}(P,P^0_{h,s,a})=(1/2)||P-P^0_{h,s,a}||_{1}$.
\item \textbf{Chi-squared uncertainty set}: Let $\mathcal{P}^{\chi}=\bigotimes_{(h,s,a)\in [H]\times S\times A}\mathcal{P}^{\chi}_{s,a}$  be the uncertainty set defined in (\ref{eq:uncertainty}) with chi-squared distance \cite{xu2023improved}
$D_{\chi}(P,P^0_{h,s,a})=\sum_{s^{\prime}}\dfrac{(P(s^{\prime})-P_{h,s,a}^0(s^{\prime}))^2}{P_{h,s,a}^0(s^{\prime})}$.
\item \textbf{KL-uncertainty set}: Let $\mathcal{P}^{KL}=\bigotimes_{(h,s,a)\in  [H]\times S\times A}\mathcal{P}^{KL}_{h,s,a}$  be the uncertainty set defined in (\ref{eq:uncertainty}) with KL-divergence metric \cite{kullback1951information}
$D_{KL}(P,P_{h,s,a}^0)=\sum_{s^{\prime}}P(s^{\prime})\log\left(\dfrac{P(s^{\prime})}{P^0_{h,s,a}(s^{\prime})}\right)$.
\end{enumerate}

\textbf{Generative Model}:
We do not know the uncertainty set, rather, we assume that we have access to a generative model or a simulator where the agent submits a query $(h,s,a)\in [H]\times S \times A$, and receives $s^{\prime}\sim P^0_{h,s,a}(\cdot)$, $r_h(s,a)$, and $g_h(s,a)$ for given $h$. Accessing the generative model  is a common assumption even for unconstrained robust MDP \cite{xu2023improved,panaganti2022sample,yang2019sample}, and constrained MDP \cite{vaswani2022near}. For example, for CMDP, \cite{vaswani2022near} use generative model to characterize the sample complexity bound. Only, recently, this has been relaxed for CMDP \cite{liu2026near}   Removing this assumption is highly nontrivial in the robust constrained setting and constitutes an important direction for future work.

\textbf{Robust Bellman Consistency equation}: Directly applying the result from \cite{iyengar2005robust}, we have for any $\pi$, for $j=r,g$, and for all $s$,
\begin{align}\label{eq:robust_bellman}
V^{\pi}_{j,h}(s)=\sum_a \pi(a|s)[j_h(s,a)+ L_{\cP_{h,s,a}}V_{j,h+1}^{\pi}]
\end{align}
where 
$L_{\cP_{h,s,a}}V=\inf\{PV:P\in \cP_{h,s,a}\}$.
\subsection{Markovian Policy in the original state-space can be sub-optimal}
For the unconstrained case, the Markovian policy is optimal for rectangular uncertainty set. However, in this counter example, we show that Markovian policy may no longer be optimal for the RCMDP  even for the rectangular uncertainty sets. 

\begin{lemma}\label{lem:sub_optimality}
Markovian policies in the original state space can be strictly sub-optimal for
RCMDPs, even under rectangular uncertainty sets.
\end{lemma}
The proof is in Appendix~\ref{app:suboptimality}, for a KL uncertainty set; the
same construction works for any set that preserves the nominal support, such
as $\chi^2$. The optimal action at the final step depends on how much utility
has already been collected, which the current state does not reveal. This
dependence is absent in unconstrained robust MDPs, and it explains the qualitative difference
between the two settings. 
Hence existing RCMDP methods~\cite{Epigraph,ghosh2024sample,ganguly_neurips,ma2025rectified},
which optimize over policies that are Markov in the original state, can be
strictly sub-optimal. We do not claim that their results are wrong, only that
their policy class is restrictive.
 


\section{Augmented RCMDP}
\label{sec:thm1}
 
Lemma~\ref{lem:sub_optimality} shows that policies that are Markov in the
original state can be sub-optimal for the RCMDP: the optimal action depends on
how much utility has already been incurred. Following
\cite{bauerle2011markov,wang2024reductions}, we recover a Markov structure by
appending the remaining budget to the state. Augmenting the state, however,
also changes what the adversary can react to. The $(s,a)$-rectangular
uncertainty model of~(\ref{eq:uncertainty}) must therefore be restated on the
augmented space before any dynamic programming argument can be used. We first
explain why.
 
\paragraph{Why the uncertainty model must be redefined.}
In~(\ref{eq:uncertainty}), nature selects a kernel
$P_h(\cdot\mid s,a)\in\cP_{h,s,a}$ independently for each $(h,s,a)$. This
admits two readings. A \emph{static} adversary commits to a single
$P\in\cP$ before the episode and uses it at every visit. A \emph{dynamic}
adversary may reselect its kernel at every step as a function of the history.
 For robust unconstrained MDPs, the distinction is immaterial. Under $(s,a)$-rectangularity,
an optimal robust policy exists that is Markov and deterministic in $s$. For
any such policy, the static and dynamic worst cases coincide, so both
adversaries yield the same optimal worst-case
value~\cite{iyengar2005robust,nilim2003robustness}. This is why prior robust MDP and
RCMDP work does not need to distinguish them.
 
\emph{This equivalence does not carry over to the RCMDP.} By
Lemma~\ref{lem:sub_optimality}, optimal policies for the RCMDP may need to
depend on the history, through the utility incurred so far. For such policies,
a dynamic adversary can tailor its kernel to the realized history, whereas a
static adversary cannot. The two worst-case values of the same policy can
therefore differ: Appendix~\ref{app:static_dynamic} gives a policy that meets
the constraint with equality under a static adversary and violates it under a
dynamic one. The adversary model is thus part of the problem definition for
RCMDPs and must be stated explicitly.
 
\paragraph{Adversary model.}
We adopt the \emph{dynamic} adversary model: throughout, $\min_P$ denotes
a dynamic adversary in the sense of~\cite{iyengar2005robust}. Write
$\tau_h=(s_1,a_1,\dots,s_{h-1},a_{h-1},s_h)$ for the history up to the state at
step $h$. At each step $h$, after observing $\tau_h$ and the action $a_h$,
nature selects $P_h(\cdot\mid\tau_h,a_h)\in\cP_{h,s_h,a_h}$. We denote the set of such strategies
by $\Xi$. The reward and utility adversaries are chosen independently.
Robustness of the constraint is defined on the total utility incurred over the
episode:
\begin{align}\label{eq:total_constraint}
  \min_{\xi\in\Xi}\ \mathbb{E}^{\pi,\xi}\Big[\sum_{h=1}^{H} g_h(s_h,a_h)\Big]\ \ge\ b .
\end{align}
Nature still draws each transition from the same per-$(h,s,a)$ sets
$\cP_{h,s,a}$ as in~(\ref{eq:uncertainty}); only its ability to adapt its choice
to the history is made explicit. A static adversary is a special case of a
dynamic one (it ignores the history). Hence, for every policy $\pi$, every
$j\in\{r,g\}$ and every fixed $P\in\cP$,
\begin{align}\label{eq:dyn_stat}
  \min_{\xi\in\Xi}V^{\pi,\xi}_{j}\ \le\ \min_{P'\in\cP}V^{\pi,P'}_{j}\ \le\ V^{\pi,P}_{j}.
\end{align}
Guarantees stated against the dynamic adversary are therefore conservative. In
particular, a constraint certified against it also holds under the static
reading and under any single model in $\cP$, including the deployed one.
 
\paragraph{Augmented state and utility.}
We consider the augmented RCMDP $(\cS^{\aug},\cA,\mathbb{P}^{\aug},H,r,g^{\aug})$
with $\cS^{\aug}=\cS\times[-H,H]$, where the second coordinate is the remaining
budget
\[
  c_1=b,\qquad c_{h+1}=c_h-g_h(s_h,a_h),\qquad
  c_h=b-\sum_{h'=1}^{h-1}g_{h'}(s_{h'},a_{h'})\in[-H,H].
\]
The augmented utility is zero for $h\le H$. It pays the total incurred utility
at the terminal step: $g^{\aug}_h(s,c,a)=0$ for $h\le H$;
$g^{\aug}_{H+1}(s,c)=-c$, so that
$g^{\aug}_{H+1}(s_{H+1},c_{H+1})=\sum_{h=1}^{H}g_h(s_h,a_h)-b$.
 
\paragraph{Augmented uncertainty set.}
Since $g_h$ is deterministic and known, uncertainty affects only the physical
coordinate $s'$; the budget coordinate moves deterministically. Because $c_h$
is a function of $\tau_h$, a dynamic adversary may select its kernel separately
for each value of $c$. The induced augmented uncertainty set is
\begin{align}\label{eq:aug_uncertainty}
  & \cP^{\aug}=\bigotimes_{h,(s,c),a}\cP^{\aug}_{h,(s,c),a},\nonumber\\ 
  & \cP^{\aug}_{h,(s,c),a}=\Big\{P^{\aug}:\
  P^{\aug}\big(s',c'\mid s,c,a\big)=P(s')\,\mathbbm{1}\{c'=c-g_h(s,a)\},\
  P\in\cP_{h,s,a}\Big\}.
\end{align}
That is, each augmented state-action pair $(s,c,a)$ draws its kernel from the
original set $\cP_{h,s,a}$, independently across values of $c$. By
construction, $\cP^{\aug}$ is $((s,c),a)$-rectangular on
$\cS^{\aug}\times\cA$. This is the property the dynamic programming and the
sample-complexity analysis below rely on. It uses the same per-$(h,s,a)$ sets
as~(\ref{eq:uncertainty}), so it introduces no new modeling assumption about
the physical dynamics. It also means the dual representations
of~(\ref{eq:uncertainty}) apply verbatim on each budget slice
$V(\cdot,c-g_h(s,a))$.
 
\paragraph{Policies and robust Bellman equations.}
Let
\[
  \Pi^{\aug}_{\mathrm M}=\big\{\pi=\{\pi_h\}_{h=1}^H:\ \pi_h(\cdot\mid s,c)\in\Delta(\cA)\
  \forall h\in[H],\,s\in\cS,\,c\in[-H,H]\big\}.
\]
For $\pi\in\Pi^{\aug}_{\mathrm M}$, with $V^{\pi}_{r,H+1}\equiv 0$ and
$V^{\pi}_{g,H+1}(s,c)=-c$,
\begin{align*}
  Q^{\pi}_{r,h}(s,c,a)&=r_h(s,a)+\min_{P\in\cP_{h,s,a}}\textstyle\sum_{s'}P(s')\,
      V^{\pi}_{r,h+1}\big(s',c-g_h(s,a)\big),\\
  Q^{\pi}_{g,h}(s,c,a)&=\min_{P\in\cP_{h,s,a}}\textstyle\sum_{s'}P(s')\,
      V^{\pi}_{g,h+1}\big(s',c-g_h(s,a)\big),\\
  V^{\pi}_{j,h}(s,c)&=\textstyle\sum_{a}\pi_h(a\mid s,c)\,Q^{\pi}_{j,h}(s,c,a),
  \qquad j\in\{r,g\}.
\end{align*}
By the rectangularity of~(\ref{eq:aug_uncertainty}), these recursions equal the
worst case over $\xi\in\Xi$~\cite{iyengar2005robust}, and the reward and utility
minimizations are carried out separately.
 
\paragraph{The constraint on the augmented state.}
The terminal utility converts the budget into the constraint. The following
lemma, proved in Appendix~\ref{app:prelim}, makes precise that the
augmented utility value measures the worst-case \emph{total} incurred utility.
 
\begin{lemma}\label{lem:total_cost}
For every (possibly history-dependent) policy $\pi$, every $h\in[H+1]$, and
every history $\tau_h$ ending at $(s,c)$,
\[
  V^{\pi}_{g,h}(\tau_h)=\min_{\xi\in\Xi}\mathbb{E}^{\pi,\xi}\Big[\sum_{h'=h}^{H}
  g_{h'}(s_{h'},a_{h'})\ \Big|\ \tau_h\Big]-c .
\]
In particular, $V^{\pi}_{g,1}(s_1,b)\ge 0$ if and only if $\pi$
satisfies~(\ref{eq:total_constraint}). More generally,
$V^{\pi}_{g,h}(\tau_h)\ge 0$ if and only if the utility already incurred plus
the worst-case remaining utility is at least $b$.
\end{lemma}
 
The RCMDP is therefore equivalent to
\begin{align}\label{eq:aug_prob}
  V^{*}_{r,1}(s_1,b)=\max_{\pi}\ V^{\pi}_{r,1}(s_1,b)
  \quad\text{s.t.}\quad V^{\pi}_{g,1}(s_1,b)\ge 0 ,
\end{align}
where the maximum is over all history-dependent policies. We call
(\ref{eq:aug_prob}) the \emph{expected-constraint} problem.
\section{Per-step decision-making}
\label{sec:perstep}

By Lemma~\ref{lem:total_cost}, $V^{\pi}_{g,h}(s_h,c_h)\ge 0$ says that, at
step $h$, the utility already incurred plus the worst-case remaining utility
still meets the threshold $b$. The issue is in finding the optimal policy in (\ref{eq:aug_prob}) as it inherently depends on the total utility incurred so far. \cite{mcmahan2024anytime} has illustrated the issue. In order to simplify, and strengthen the feasibility, we consider
\emph{per-step} decision-making which we describe next. 

\paragraph{Reachable successors and the feasible region.}
For $(h,s,a)$ let
$\cN_h(s,a)=\bigcup_{P\in\cP_{h,s,a}}\mathrm{supp}(P)$ denote the successors the
adversary can reach (for TV balls, $\cN_h(s,a)=\cS$). Define backward
\begin{align*}
  \cF_{H+1}&=\{(s,c):\,c\le 0\},\\
  \cA_h(s,c)&=\big\{a\in\cA:\ (s',\,c-g_h(s,a))\in\cF_{h+1}\ \ \forall s'\in\cN_h(s,a)\big\},\\
  \cF_h&=\{(s,c):\,\cA_h(s,c)\neq\emptyset\}.
\end{align*}
$\cF_h$ and $\cA_h$ depend only on $g$, $b$, and the supports of $\cP$, not on
the transition probabilities.
 
\begin{definition}[Per-step feasibility]\label{def:perstep}
A history $\tau_h=(s_1,a_1,\dots,s_h)$ is \emph{reachable} under $\pi$ if it
has positive probability under $\pi$ and some adversary $\xi\in\Xi$;
equivalently, $\pi_k(a_k\mid\tau_k)>0$ and $s_{k+1}\in\cN_k(s_k,a_k)$ for every
$k<h$. A (possibly history-dependent) policy $\pi$ is \emph{per-step feasible}
if $V^{\pi}_{g,h}(\tau_h)\ge 0$ for every $h\in[H+1]$ and every history
$\tau_h$ reachable under $\pi$ from $(s_1,b)$. We denote this class by
$\Pi_{\mathrm{step}}$.
\end{definition}
 
By Lemma~\ref{lem:total_cost}, per-step feasibility says that at every
decision point the agent can still guarantee, in the worst case, that its
total utility reaches $b$. The next lemma gives two equivalent descriptions:
(ii) the threshold is met on every trajectory the adversary can produce, and
(iii) the policy starts in $\cF_1$ and only ever plays actions in $\cA_h$.
Condition (iii) is the one used by dynamic programming and in our proofs.
 
\begin{lemma}\label{lem:perstep_char}
For a (possibly history-dependent) policy $\pi$, the following are equivalent:
\begin{enumerate}
\item[(i)] $\pi\in\Pi_{\mathrm{step}}$;
\item[(ii)] $\sum_{h=1}^H g_h(s_h,a_h)\ge b$ along every trajectory that is
reachable under $\pi$;
\item[(iii)] $(s_1,b)\in\cF_1$, and at every reachable history $\tau_h$ with
$h\le H$, the distribution $\pi_h(\cdot\mid\tau_h)$ is supported on
$\cA_h(s_h,c_h)$.
\end{enumerate}
\end{lemma}
\begin{proof}
Along any trajectory, $-c_{H+1}=\sum_{h=1}^Hg_h(s_h,a_h)-b$, so (ii) says
exactly that $c_{H+1}\le0$ on every reachable trajectory.
 
\emph{(i)$\Rightarrow$(ii).} Take $h=H+1$ in Definition~\ref{def:perstep}. A
reachable history $\tau_{H+1}$ is a reachable trajectory, and
$V^{\pi}_{g,H+1}(\tau_{H+1})=-c_{H+1}$. Hence $-c_{H+1}\ge0$ on every reachable
trajectory.
 
\emph{(ii)$\Rightarrow$(i).} Fix a reachable history $\tau_h$. Every
trajectory that extends $\tau_h$ with positive probability, under $\pi$ and
any adversary, is reachable, so $-c_{H+1}\ge0$ on it by (ii). By
Lemma~\ref{lem:total_cost}, $V^{\pi}_{g,h}(\tau_h)$ is a worst-case conditional
expectation of $-c_{H+1}$ over such extensions, and is therefore
nonnegative.
 
\emph{(iii)$\Rightarrow$(ii).} By forward induction, every reachable
$(s_h,c_h)$ lies in $\cF_h$. This holds at $h=1$ by assumption. If
$(s_h,c_h)\in\cF_h$ and $\pi$ plays $a\in\cA_h(s_h,c_h)$, then by the definition
of $\cA_h$ every successor the adversary can reach,
$(s',c_h-g_h(s_h,a))$ with $s'\in\cN_h(s_h,a)$, lies in $\cF_{h+1}$. At
$h=H+1$ this gives $c_{H+1}\le0$ on every reachable trajectory.
 
\emph{(ii)$\Rightarrow$(iii).} We first show the following claim: \emph{from
any $(s,c)\notin\cF_k$ at step $k$, whatever the policy does, the adversary
can reach a terminal state with $c_{H+1}>0$.} We argue by backward induction on
$k$. For $k=H+1$, $(s,c)\notin\cF_{H+1}$ means $c>0$. For $k\le H$,
$(s,c)\notin\cF_k$ means $\cA_k(s,c)=\emptyset$, so every action $a$ has some
successor $s'\in\cN_k(s,a)$ with $(s',c-g_k(s,a))\notin\cF_{k+1}$. The dynamic
adversary can choose a $P\in\cP_{k,s,a}$ with $P(s')>0$, and the claim at $k+1$
applies from there.
Now suppose (iii) fails. If $(s_1,b)\notin\cF_1$, the claim at $k=1$
contradicts (ii). Otherwise, at some reachable history $\tau_h$, $\pi$ plays
some $a\notin\cA_h(s_h,c_h)$ with positive probability. Some
$s'\in\cN_h(s_h,a)$ then has $(s',c_h-g_h(s_h,a))\notin\cF_{h+1}$, and the
adversary can reach it. The claim at $k=h+1$ again contradicts (ii).
\end{proof}
 
The per-step RCMDP is
\begin{align}\label{eq:step_solve}
  V^{\diamond}_{r,1}(s_1,b)=\max_{\pi\in\Pi_{\mathrm{step}}}\ V^{\pi}_{r,1}(s_1,b),
\end{align}
which is well posed if and only if $(s_1,b)\in\cF_1$. Throughout the rest of
the paper, we assume that it is.
 
\begin{assumption}[Per-step feasibility]\label{ass:feas}
$(s_1,b)\in\cF_1$; equivalently, some policy meets the threshold along every
trajectory the adversary can induce.
\end{assumption}
The weaker condition that some policy satisfies the \emph{expected}
constraint is not enough for our results. 
 
\begin{proposition}[Per-step dynamic programming]\label{prop:perstep_dp}
Let $V^{\diamond}_{r,H+1}\equiv 0$ and, for $(s,c)\in\cF_h$,
\[
  V^{\diamond}_{r,h}(s,c)=\max_{a\in\cA_h(s,c)}\Big\{r_h(s,a)
  +\min_{P\in\cP_{h,s,a}}\textstyle\sum_{s'}P(s')\,
  V^{\diamond}_{r,h+1}\big(s',c-g_h(s,a)\big)\Big\}.
\]
A deterministic Markov policy $\pi^{\diamond}\in\Pi^{\aug}_{\mathrm M}$ selecting
a maximizer is optimal for~(\ref{eq:step_solve}) among all history-dependent
policies in $\Pi_{\mathrm{step}}$.
\end{proposition}
\begin{proof}
By Lemma~\ref{lem:perstep_char}, per-step feasibility is exactly the
state-dependent action restriction $a\in\cA_h(s,c)$ on $\cF_h$, and every
successor of such an action lies in $\cF_{h+1}$. Problem~(\ref{eq:step_solve})
is therefore an unconstrained robust MDP on $\{\cF_h\}$ with action sets
$\cA_h$ and the rectangular set $\cP^{\aug}$. The claim follows
from~\cite{iyengar2005robust}.
\end{proof}
 
\begin{remark}[Randomization]\label{rem:randomization}
$\Pi_{\mathrm{step}}$ contains randomized policies, but randomization does not
improve the value of~(\ref{eq:step_solve}). One might try to satisfy the
constraint at $(s,c)$ by mixing an action $a$ with
$Q^{\pi}_{g,h}(s,c,a)<0$ and a safe action $a'$ so that
$\langle p,Q^{\pi}_{g,h}(s,c,\cdot)\rangle\ge 0$. This is not per-step
feasible. Since
$Q^{\pi}_{g,h}(s,c,a)=\min_{P\in\cP_{h,s,a}}\sum_{s'}P(s')V^{\pi}_{g,h+1}(s',c-g_h(s,a))<0$,
the minimizing $P$ places positive mass on some $s'$ with
$V^{\pi}_{g,h+1}(s',c-g_h(s,a))<0$. Because $p(a)>0$, that successor is
reachable under a dynamic adversary, so Definition~\ref{def:perstep} is
violated at step $h+1$. Mixing repairs the average at step $h$ only by moving
the violation to step $h+1$ on the branch where $a$ is played. Consequently, by
Lemma~\ref{lem:perstep_char}(iii), every action in the \emph{support} of
$\pi_h(\cdot\mid s,c)$ must lie in $\cA_h(s,c)$. The
feasible set at $(s,c)$ is therefore the face $\Delta(\cA_h(s,c))$ of the
simplex. A linear objective over this face attains its maximum at a vertex,
which is why $\pi^{\diamond}$ may be taken deterministic. This differs from the
empirical problem~(\ref{eq:lp}) solved by Algorithm~\ref{algo:rs-cmdp}. There,
the constraint is on the \emph{average}
$\langle p,\hat Q_{g,h}(s,c,\cdot)\rangle$, so mixing a high-reward action that
violates the constraint with a safe action can be strictly better, and the
optimal $p$ may be randomized (supported on at most two actions). The two
problems are deliberately different: at every $(s,c)\in\cF_h$, the per-step
feasible set $\Delta(\cA_h(s,c))$ is contained in the feasible set
of~(\ref{eq:lp}) (Lemma~\ref{lem:joint}(a)). Problem~(\ref{eq:lp}) is thus a
relaxation of the per-step problem, and the analysis uses only this
containment.
\end{remark}
 
\begin{remark}
$\Pi_{\mathrm{step}}$ is contained in the feasible set of~(\ref{eq:aug_prob}),
so $V^{\diamond}_{r,1}\le V^{*}_{r,1}$. The per-step problem is the more
conservative one: for full-support sets such as TV, it requires the threshold
to be met along every trajectory.
\end{remark}

\paragraph{Computational complexity}
The computational cost depends on the size of the augmented state space. To
keep it finite, we first assume the following.
\begin{assumption}\label{assum:discrete}
The utilities and the threshold take values on a finite grid, so that the
remaining budget takes values in a finite set $\cC\subset[-H,H]$.
\end{assumption}
Under Assumption~\ref{assum:discrete}, RCVI operates on
$\cS\times\cC$. Section~\ref{sec:extension} removes the assumption by
discretizing $[-H,H]$ at resolution $\epsilon/(3(H+1))$, which gives
$|\cC|=\cO(H^2/\epsilon)$. The sample complexity changes only through
$\log|\cC|$, and the computational cost remains polynomial in $H$ and
$1/\epsilon$.


\section{Algorithm}
\label{sec:algorithm}

We now describe the Robust Constrained Value Iteration (RCVI) algorithm
(Algorithm~\ref{algo:rs-cmdp}). RCVI is the empirical counterpart of the
time-consistent recursion of Section~\ref{sec:tc}. It runs backward over
$h=H,\dots,1$ on the augmented state space. At each step it
(i) estimates the nominal kernel from fresh samples,
(ii) computes worst-case $Q$-functions for the reward and the utility under the
empirical uncertainty set, and
(iii) solves a single-constraint linear program at every augmented state.

\paragraph{Sampling and the empirical uncertainty set.}
We do not know the nominal model and instead query the generative model. At
step $h$, for every $(s,a)$, we draw $N$ \emph{fresh} samples
$s'_1,\dots,s'_N\sim P^0_h(\cdot\mid s,a)$ (Line~4) and form the empirical
kernel
\[
  \hat P^0_h(s'\mid s,a)=\frac{N_h(s,a,s')}{N},\qquad
  N_h(s,a,s')=\sum_{i=1}^{N}\mathbbm{1}\{s'_i=s'\}
\]
(Line~5). The empirical uncertainty set is
\[
  \hat\cP=\bigotimes_{h,s,a}\hat\cP_{h,s,a},\qquad
  \hat\cP_{h,s,a}=\big\{P\in\Delta(\cS):\ D\big(P,\hat P^0_h(\cdot\mid s,a)\big)\le\rho\big\},
\]
where $D$ is the TV distance, the $\chi^2$ divergence, or the KL divergence. The
total number of samples is $N_{\mathrm{tot}}=N|\cS||\cA|H$. Because the samples
used at step $h$ are drawn afresh, they are independent of the value functions
$\hat V_{j,h+1}$ computed at later steps. This is what allows the concentration
bounds of Section~\ref{sec:analysis} to be applied to data-dependent value
functions without a covering argument over value functions.

\paragraph{Worst-case $Q$-functions.}
Let $\Gamma:\R\to\cC$ round a budget \emph{down} to the grid $\cC\subset[-H,H]$,
clamping at $-H$ (Appendix~\ref{sec:continuous}). The tracked budget evolves as $\hat c_{h+1}=\Gamma(\hat c_h-g_h(s_h,a_h))$, so
$\hat c_h\le c_h$, and each step credits at most $\epsilon_0$ (the grid
resolution) more utility than was incurred. Two consequences follow. Every
trajectory that meets the threshold in the original problem also meets it on
the grid, so the per-step benchmark is preserved. The violation incurred in the
original problem exceeds that on the grid by at most $(H+1)\epsilon_0$
(Appendix~\ref{sec:continuous}). For each $(s,\hat c,a)$, RCVI computes
(Lines~7--8)
\begin{align*}
  \hat Q_{g,h}(s,\hat c,a)&=\cL_{\hat\cP_{h,s,a}}\hat V_{g,h+1}\big(\cdot,\Gamma(\hat c-g_h(s,a))\big),\\
  \hat Q_{r,h}(s,\hat c,a)&=r_h(s,a)+\cL_{\hat\cP_{h,s,a}}\hat V_{r,h+1}\big(\cdot,\Gamma(\hat c-g_h(s,a))\big),
\end{align*}
where $\cL_{\hat\cP_{h,s,a}}V=\inf_{P\in\hat\cP_{h,s,a}}\E_{s'\sim P}[V(s')]$.
Since the budget coordinate evolves deterministically, the worst case is taken
only over the physical successor $s'$, on the budget slice
$\Gamma(\hat c-g_h(s,a))$. For $f$-divergence sets the inner minimization has a
one-dimensional dual. For TV, following Proposition~1
of~\cite{panaganti2022sample}, for $V:\cS\to[0,V_{\max}]$,
\[
  \cL_{\hat\cP^{\mathrm{TV}}_{h,s,a}}V
  =-\inf_{\eta\in[0,2V_{\max}/\rho]}\Big\{\E_{s'\sim\hat P^0_h(\cdot\mid s,a)}\big[(\eta-V(s'))_+\big]
  +\rho\,\big(\eta-\min_{s''\in\cS}V(s'')\big)_+-\eta\Big\},
\]
and analogous duals hold for $\chi^2$ and KL (Appendices~\ref{proof:chi} and~\ref{proof:kl}). The
reward values lie in $[0,H]$ and the utility values in $[-H,H]$
(Appendix~\ref{app:prelim}), so the clipping in Lines~7--8 never binds.

\paragraph{Per-state decision: a single-constraint LP.}
At every $(s,\hat c)\in\cS\times\cC$, RCVI imposes the robust total-utility
constraint at that decision point, given the already computed continuation, and
maximizes reward (Line~11):
\begin{align}\label{eq:lp}
  \max_{p\in\Delta(\cA)}\ \big\langle p,\hat Q_{r,h}(s,\hat c,\cdot)\big\rangle
  \quad\text{s.t.}\quad
  \big\langle p,\hat Q_{g,h}(s,\hat c,\cdot)\big\rangle\ \ge\ -\kappa ,
\end{align}
with a constant slack $\kappa\ge 0$. By Lemma~\ref{lem:total_cost}, the left-hand
side of the constraint is the empirical worst-case total utility (already
incurred plus remaining) minus $b$. Problem~(\ref{eq:lp}) is feasible if and
only if $\max_a\hat Q_{g,h}(s,\hat c,a)\ge-\kappa$. It then has an optimal
solution supported on at most two actions, computable in
$O(|\cA|\log|\cA|)$ time from the upper concave envelope of the points
$\{(\hat Q_{g,h}(s,\hat c,a),\hat Q_{r,h}(s,\hat c,a))\}_{a\in\cA}$.

If~(\ref{eq:lp}) is infeasible, no action distribution meets the constraint
from $(s,\hat c)$, and RCVI plays the least-violating action,
$\hat\pi_h(\cdot\mid s,\hat c)=\mathbbm{1}\{a=\hat a\}$ with
$\hat a\in\arg\max_a\hat Q_{g,h}(s,\hat c,a)$. This fallback does not affect our
guarantees, which hold for any fallback rule. 

Finally, the value functions at step $h$ are updated as inner products with the
chosen action distribution (Line~12), so every quantity RCVI computes is the
empirical robust value of the policy it outputs.

\paragraph{Role of the slack.}
The slack $\kappa$ absorbs estimation error in the \emph{violation} bound only.
The reward guarantee needs no slack, because the comparison class is per-step
feasible (Lemma~\ref{lem:joint}). A constant $\kappa$ suffices, and choosing
$\kappa=\epsilon/2$ gives total violation at most $\epsilon$
(Theorems~\ref{thm:tv}--\ref{thm:kl}).

\begin{remark}[Why the LP rather than the masked recursion]\label{rem:why_lp}
Proposition~\ref{prop:perstep_dp} solves the per-step problem exactly by
restricting the policy to the safe actions $\cA_h(s,c)$ and maximizing reward
over them. RCVI instead solves the relaxation~(\ref{eq:lp}), for four reasons.
\emph{(i) It targets the actual constraint.} Problem~(\ref{eq:cmdp_robust})
constrains the \emph{expected} total utility. The per-step class serves only as
the benchmark against which reward is certified.
\emph{(ii) It is never worse and can be much better.} Every per-step feasible
action distribution is feasible for~(\ref{eq:lp}), by Lemma~\ref{lem:joint}(a).
In the instance of Appendix~\ref{app:example}, the masked recursion earns reward
$0$, whereas the LP meets the robust expected constraint with equality and earns
$1/2$.
\emph{(iii) It needs no support identification.} $\cA_h$ depends on the
reachable sets $\cN_h(s,a)$, which for KL and $\chi^2$ are the supports of the
unknown nominal kernel. An estimated support can miss a rare successor. The LP
uses only estimated values.

\end{remark}

\setlength{\textfloatsep}{0pt}
\begin{algorithm}[tbh]
\caption{RCVI: Robust Constrained Value Iteration}
\label{algo:rs-cmdp}
\textbf{Input:} budget grid $\cC$ with round-down map $\Gamma$; slack $\kappa\ge0$;
samples per $(h,s,a)$ $N$; radius $\rho$; divergence $D$.
\begin{algorithmic}[1]
\STATE For all $(s,\hat c)\in\cS\times\cC$: $\hat V_{g,H+1}(s,\hat c)\gets-\hat c$,\ \ $\hat V_{r,H+1}(s,\hat c)\gets 0$.
\FOR{$h=H,H-1,\dots,1$}
  \FOR{all $(s,a)\in\cS\times\cA$}
    \STATE Draw $N$ fresh samples $s'_1,\dots,s'_N\sim P^0_h(\cdot\mid s,a)$ from the generative model.
    \STATE $\hat P^0_h(s'\mid s,a)\gets \frac{1}{N}\sum_{i=1}^{N}\mathbb{1}\{s'_i=s'\}$;\ \ $\hat\cP_{h,s,a}\gets\{P:\ D(P,\hat P^0_h(\cdot\mid s,a))\le\rho\}$.
    \FOR{all $\hat c\in\cC$}
      \STATE $\hat Q_{g,h}(s,\hat c,a)\gets \mathrm{clip}_{[-H,H]}\Big(\cL_{\hat\cP_{h,s,a}}\hat V_{g,h+1}\big(\cdot,\Gamma(\hat c-g_h(s,a))\big)\Big)$
      \STATE $\hat Q_{r,h}(s,\hat c,a)\gets \mathrm{clip}_{[0,H]}\Big(r_h(s,a)+\cL_{\hat\cP_{h,s,a}}\hat V_{r,h+1}\big(\cdot,\Gamma(\hat c-g_h(s,a))\big)\Big)$
    \ENDFOR
  \ENDFOR
  \STATE For all $(s,\hat c)\in\cS\times\cC$: \textbf{if} $\max_a\hat Q_{g,h}(s,\hat c,a)\ge-\kappa$ \textbf{then} $\hat\pi_h(\cdot\mid s,\hat c)\gets$ a solution of~(\ref{eq:lp});
         \textbf{else} $\hat\pi_h(\cdot\mid s,\hat c)\gets\mathbbm{1}\{a=\hat a\}$ with $\hat a\in\arg\max_a\hat Q_{g,h}(s,\hat c,a)$.
  \STATE For all $(s,\hat c)\in\cS\times\cC$ and $j\in\{r,g\}$: $\hat V_{j,h}(s,\hat c)\gets\big\langle\hat\pi_h(\cdot\mid s,\hat c),\hat Q_{j,h}(s,\hat c,\cdot)\big\rangle$.
\ENDFOR
\STATE \textbf{Output} $\hat\pi=\{\hat\pi_h\}_{h=1}^{H}$, executed on $(s_h,\hat c_h)$ with the tracked budget $\hat c_1=\Gamma(b)$ and $\hat c_{h+1}=\Gamma(\hat c_h-g_h(s_h,a_h))$.
\end{algorithmic}
\end{algorithm}

\section{Results}
\subsection{Main Results}
\label{sec:main_results}

Throughout this section, $\mathrm{Sub\text{-}Opt}$ and $\mathrm{Violation}$
are as in Definition~\ref{def:sample_complexity}, with the benchmark
$V^{\diamond}_{r,1}(s_1,b)$ of Eq.~(\ref{eq:step_solve}) and the dynamic
adversary of Section~\ref{sec:thm1}; equivalently,
\[
  \mathrm{Violation}(\hat\pi)=b-\min_{\xi\in\Xi}\E^{\hat\pi,\xi}\Big[\sum_{h=1}^{H}g_h\Big].
\]
The violation is measured against the robust \emph{expected} total-utility
constraint~(\ref{eq:total_constraint}) of the original problem. We assume that
the per-step problem is feasible, i.e., $(s_1,b)\in\cF_1$
(Assumption~\ref{ass:feas}). Algorithm~\ref{algo:rs-cmdp} draws $N$ samples per
$(h,s,a)$, so the total number of samples is $N_{\mathrm{tot}}=N|\cS||\cA|H$.

\begin{theorem}[TV uncertainty]\label{thm:tv}
Let $\cP$ be the TV uncertainty set and suppose Assumption~\ref{ass:feas} holds.
Run Algorithm~\ref{algo:rs-cmdp} with slack $\kappa=\epsilon/2$ and
$N_{\mathrm{tot}}\ge N_{\mathrm{TV}}$, where
\[
  N_{\mathrm{TV}}=\frac{C_1|\cS||\cA|H^5}{\epsilon^2}\,
  \log\!\Big(\frac{|\cS||\cA||\cC|H}{\epsilon\delta}\Big)
\]
for an absolute constant $C_1>0$. Then, with probability at least $1-\delta$,
\[
  \mathrm{Violation}(\hat\pi)\le\epsilon
  \qquad\text{and}\qquad
  V^{\hat\pi}_{r,1}(s_1,b)\ \ge\ V^{\diamond}_{r,1}(s_1,b)-\epsilon .
\]
\end{theorem}
The proof is in Appendix~\ref{proof:tv}.

Theorem~\ref{thm:tv} shows that
$\tilde{\cO}(|\cS||\cA|H^5/\epsilon^2)$ samples suffice for RCVI to be
$\epsilon$-feasible for the RCMDP while earning, up to $\epsilon$, at least the
reward of every per-step feasible policy, including history-dependent ones.
To our knowledge, this is the first sample-complexity guarantee for RCMDPs that
does not rely on a policy-optimization oracle. The dependence on $|\cS|$,
$|\cA|$ and $\epsilon$ matches that of robust value iteration for unconstrained
robust MDPs with a generative model~\cite{panaganti2022sample}.

The bound depends on $|\cC|$ only through the logarithm. The budget coordinate
evolves deterministically and $g$ is known, so no samples are spent on it. The
same $N$ samples at $(h,s,a)$ serve every budget slice, and $|\cC|$ enters only
through a union bound. The computational cost, in contrast, scales with $|\cC|$,
since one LP is solved per augmented state.

\begin{theorem}[$\chi^2$ uncertainty]\label{thm:chi}
Let $\cP$ be the $\chi^2$ uncertainty set and suppose
Assumption~\ref{ass:feas} holds. Run Algorithm~\ref{algo:rs-cmdp} with
$\kappa=\epsilon/2$ and $N_{\mathrm{tot}}\ge N_{\chi}$, where
\[
  N_{\chi}=\frac{C_2\,c(\rho)\,|\cS||\cA|H^5}{\epsilon^2}\,
  \log\!\Big(\frac{|\cS||\cA||\cC|H\,c(\rho)}{\epsilon\delta}\Big),
  \qquad c(\rho)=\frac{(1+\rho)^2}{\big(\sqrt{1+\rho}-1\big)^2},
\]
and $C_2>0$ is an absolute constant. Then, with probability at least $1-\delta$, the
conclusions of Theorem~\ref{thm:tv} hold.
\end{theorem}
The proof is in Appendix~\ref{proof:chi}. The factor $c(\rho)$ behaves as
$4/\rho^2$ for small $\rho$ and as $\rho$ for large $\rho$.

\begin{theorem}[KL uncertainty]\label{thm:kl}
Let $\cP$ be the KL uncertainty set, let
$\zeta=\min_{h,s,a,s':\,P^0_h(s'\mid s,a)>0}P^0_h(s'\mid s,a)$, and suppose
Assumption~\ref{ass:feas} holds. Run Algorithm~\ref{algo:rs-cmdp} with
$\kappa=\epsilon/2$ and $N_{\mathrm{tot}}\ge N_{\mathrm{KL}}$, where
\[
  N_{\mathrm{KL}}=\frac{C_3|\cS||\cA|H^5}{\rho^2\zeta^2\epsilon^2}\,
  \log\!\Big(\frac{|\cS|^2|\cA|H}{\delta}\Big)
\]
for an absolute constant $C_3>0$, and $\epsilon\le4H^2/\rho$. Then, with probability at least $1-\delta$,
the conclusions of Theorem~\ref{thm:tv} hold.
\end{theorem}
The proof is in Appendix~\ref{proof:kl}. The factor $1/(\rho^2\zeta^2)$ is
inherited from the concentration of the KL dual and also appears in the
unconstrained case~\cite{panaganti2022sample,xu2023improved}. For KL the
one-step deviation bound holds uniformly over value functions, so $|\cC|$ does
not appear even in the logarithm.

\paragraph{On the benchmark.}
The per-step class contains history-dependent policies, and by
Lemma~\ref{lem:sub_optimality} its value can strictly exceed that of every
policy that is Markov in the original state. It is, however, more conservative
than the expected constraint~(\ref{eq:total_constraint}): per-step feasibility
requires the threshold to be met along every trajectory the adversary can
induce. Hence $V^{\diamond}_{r,1}\le V^{*}_{r,1}$, where $V^{*}$ is the optimum
of~(\ref{eq:aug_prob}). RCVI is not restricted to the per-step class. Its LP
enforces the expected constraint and can earn strictly more than $V^{\diamond}$
(Remark~\ref{rem:why_lp}). {\em Whether $\tilde\cO(1/\epsilon^2)$ samples suffice to
compete with $V^{*}$ itself is open}.

For full-support sets such as TV, $\cN_h(s,a)=\cS$: the adversary can
move mass to any state, including states the nominal model never visits.
Assumption~\ref{ass:feas} then requires the threshold to be met even after
such a jump, which fails whenever some state offers too little utility, and
the reward guarantee becomes vacuous. For KL and $\chi^2$, only successors in
the nominal support are reachable, and the assumption asks only that the
threshold be attainable on every path the nominal model allows. The TV guarantee is therefore most informative when $b$ is
small relative to the utility that can be guaranteed along every path. For KL
and $\chi^2$, only successors in the nominal support are reachable, and the
benchmark is correspondingly less conservative.

\paragraph{Comparison with prior work.}
\cite{ganguly_neurips,Epigraph} establish an $\cO(1/\epsilon^4)$ iteration
complexity, where each iteration requires exact robust policy evaluation, and
they compete with the best policy that is Markov in the original state.
\cite{ghosh2024sample} obtains $\tilde\cO(1/\epsilon^2)$ sample complexity given
a policy-optimization oracle. Our guarantee is sample-based, requires no
oracle, and competes with a class that includes non-Markovian policies. The
benchmarks are not ordered in general. $V^{\diamond}$ can exceed the best
original-state Markov value (Lemma~\ref{lem:sub_optimality}), but it can also be
smaller, because per-step feasibility is stricter than the expected constraint.

\paragraph{Almost-sure safety.}
If the reachable sets $\cN_h(s,a)$ are known, replacing~(\ref{eq:lp}) with
reward maximization over $\cA_h(s,\hat c)$ yields a policy that is per-step
feasible, i.e., $\sum_h g_h\ge b$ along every trajectory under every adversary.
Its violation is then exactly zero, and the reward guarantee of
Theorems~\ref{thm:tv}--\ref{thm:kl} holds verbatim.

\paragraph{Lower bounds.}
A lower bound for RCMDPs is open. For unconstrained robust MDPs with a
generative model, minimax rates are known~\cite{shi2023curious}. Whether the
constraint, or the budget augmentation, provably increases the horizon
dependence is an interesting question.

\subsection{Analysis}
\label{sec:analysis}

\paragraph{Concentration with fresh samples.}
Let $\epsilon_N$ denote the per-step deviation bound of
Lemma~\ref{lem:diff_v}: with high probability,
$|\cL_{\cP_{h,s,a}}V-\cL_{\hat\cP_{h,s,a}}V|\le\epsilon_N$ for any $V$ that is
independent of the step-$h$ samples. Because Algorithm~\ref{algo:rs-cmdp} draws
fresh samples at every $h$, the empirical continuation values
$\hat V^{\hat\pi}_{j,h+1}$ and $\hat V^{\pi^{\diamond}}_{j,h+1}$ are independent
of the step-$h$ samples: $\hat\pi_{h+1:H}$ depends only on samples from later
steps, and $\pi^{\diamond}$ is fixed. Lemma~\ref{lem:diff_v} therefore applies
to them directly, after a union bound over $(h,s,a,\hat c)$. Only these
empirical values need to be concentrated; the true values enter through the
1-Lipschitz property below. The utility values lie in $[-H,H]$; since each $\cL$ is
translation-equivariant, we apply the lemma to $V+H\in[0,2H]$.

\paragraph{Value propagation.}
Combining this with the 1-Lipschitz property of $\cL$
(Lemma~\ref{lem:L-1lip}), backward induction gives, for
$\pi\in\{\hat\pi,\pi^{\diamond}\}$ and $j\in\{r,g\}$,
\begin{align}\label{eq:propagation}
  \big\|V^{\pi}_{j,h}-\hat V^{\pi}_{j,h}\big\|_\infty\le (H-h+1)\,\epsilon_N
\end{align}
(Lemma~\ref{lem:V-propagation}). By Line~12 of
Algorithm~\ref{algo:rs-cmdp}, the values computed by RCVI are exactly the
empirical robust values of its output, $\hat V_{j,h}=\hat V^{\hat\pi}_{j,h}$.

\paragraph{Violation bound.}
At the root, $(s_1,b)\in\cF_1$, and Lemma~\ref{lem:joint}(a) shows that the LP
at the root is feasible. Its constraint gives
$\hat V^{\hat\pi}_{g,1}(s_1,b)\ge-\kappa$. By~(\ref{eq:propagation}),
\[
  V^{\hat\pi}_{g,1}(s_1,b)\ \ge\ -\kappa-H\epsilon_N .
\]
By Lemma~\ref{lem:total_cost}, this is exactly a bound on the violation of the
robust expected total-utility constraint.

\paragraph{Reward bound.}
The key step is to compare RCVI with the benchmark policy $\pi^{\diamond}$
\emph{on the empirical model}. The joint induction of Lemma~\ref{lem:joint}
shows that, at every $(s,c)\in\cF_h$:
(a) every action in $\cA_h(s,c)$ satisfies
$\hat Q^{\hat\pi}_{g,h}(s,c,a)\ge-\kappa$, so $\pi^{\diamond}_h(\cdot\mid s,c)$
is feasible for the LP;
(b) $\hat V^{\hat\pi}_{g,h}(s,c)\ge-\kappa$; and
(c) $\hat V^{\hat\pi}_{r,h}(s,c)\ge\hat V^{\pi^{\diamond}}_{r,h}(s,c)$.
Step (a) uses only the lower bound (b) at successor states, which the LP
guarantees whatever actions RCVI chooses later. This is why the comparison is
made with per-step feasible policies: such policies never rely on a successor
state having a utility surplus. They need only that each successor meets the
threshold on its own. Combining with~(\ref{eq:propagation}),
\[
  V^{\hat\pi}_{r,1}(s_1,b)\ \ge\ \hat V^{\hat\pi}_{r,1}(s_1,b)-H\epsilon_N
  \ \ge\ \hat V^{\pi^{\diamond}}_{r,1}(s_1,b)-H\epsilon_N
  \ \ge\ V^{\diamond}_{r,1}(s_1,b)-2H\epsilon_N .
\]

\paragraph{Sample size.}
Choosing $N$ such that $H\epsilon_N\le\epsilon/2$ and setting $\kappa=\epsilon/2$
gives both bounds. For TV, $\epsilon_N=\tilde\cO(H/\sqrt N)$, so
$N=\tilde\cO(H^4/\epsilon^2)$ per $(h,s,a)$ and
$N_{\mathrm{tot}}=\tilde\cO(|\cS||\cA|H^5/\epsilon^2)$. The $\chi^2$ and KL cases
follow identically from their deviation bounds (Lemmas~\ref{lem:chi_dev}
and~\ref{lem:kl_dev}).

\subsection{Extension: Continuous Utilities}
\label{sec:extension}

\textbf{Relaxation of Assumption~\ref{assum:discrete}.}
When the utilities are not on a grid, we discretize the budget interval
$[-H,H]$ at resolution $\epsilon_0$ and track the budget with the round-down
map $\Gamma(x)=\max\{\hat c\in\cC:\ \hat c\le\max(x,-H)\}$, i.e.,
$\hat c_1=\Gamma(b)$ and $\hat c_{h+1}=\Gamma(\hat c_h-g_h)$. Then
$0\le c_h-\hat c_h\le h\,\epsilon_0$: the grid credits slightly more utility
than was actually incurred.
Consequently, (i) every policy that is per-step feasible in the original problem
remains per-step feasible on the grid, after being adapted to the grid as a
history-dependent policy, which the per-step class allows
(Lemma~\ref{lem:disc-bench}); hence the benchmark does not decrease. (ii) A policy
computed on the grid incurs at most $(H+1)\epsilon_0$ additional violation in
the original problem, and its reward is unchanged
(Lemma~\ref{lem:violation_discrete_to_original}). Choosing
$\epsilon_0=\epsilon/(3(H+1))$, $\kappa=\epsilon/3$, and $N$ such that
$H\epsilon_N\le\epsilon/3$ preserves the guarantees of
Theorems~\ref{thm:tv}--\ref{thm:kl} with $|\cC|=\cO(H^2/\epsilon)$. The sample complexity is unchanged up to the
logarithm, and the computational cost remains polynomial in $H$ and $1/\epsilon$.




\section{Experiments}
In this section, we present the empirical results obtained from experiments on two benchmarks \cite{ganguly_neurips}. In both the environments, we use KL divergence uncertainty sets. The details of which are in Appendix~\ref{kl_divergence}. We compare our approach with (i) Robust constrained Natural Policy gradient (RNPG) \cite{ganguly_neurips}, (ii) Robust-CRPO (robust version of CRPO \cite{xu2021crpo}),  (iii) RRPO \cite{ma2025rectified} and (iv) EPIRC-PGS \cite{Epigraph}.

\textbf{Constrained River-swim:} The CRS comprises of six states, corresponding to six islands. At each state, the agent selects between two actions: \emph{swim left} (\(a_0\)) or \emph{swim right} (\(a_1\)). Rewards are assigned only at the boundary states, with cost scores associated at each state. As state progresses cost increases. The objective is to maximize cumulative rewards while keeping the cost below a certain threshold (refer to appendix \ref{app:experiments} for more details). 

\begin{figure*}[h!]
    \centering
    \begin{subfigure}[h]{0.5\textwidth}
        \centering
        \includegraphics[clip=2in 0in 2in 0in,height=2.3in]{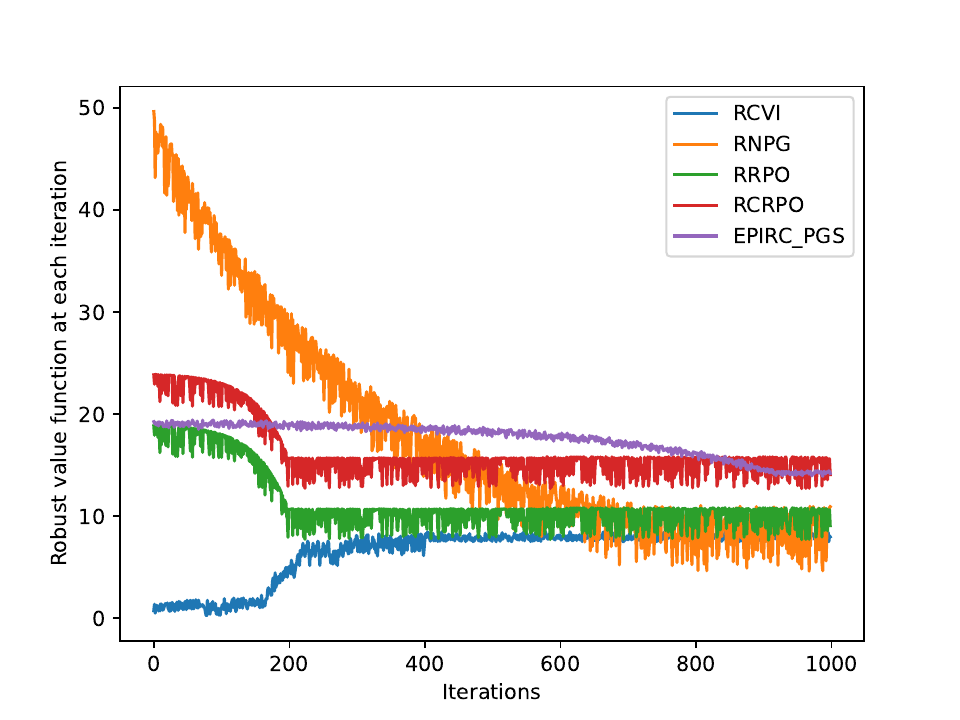}
         \caption{The value function at each iteration step}
        \label{fig:crs_rvi_main}
    \end{subfigure}%
    ~ 
    \begin{subfigure}[h]{0.5\textwidth}
        \centering
        \includegraphics[clip=2in 2in 2in 0in,height=2.3in]{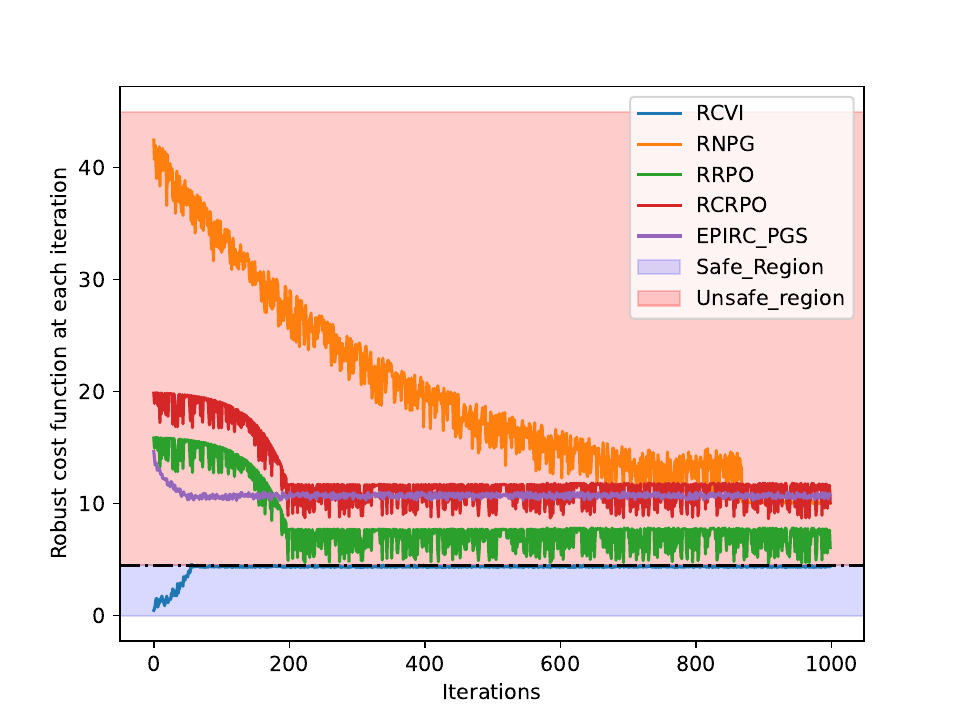}
        \caption{The cost function at each iteration step}
        \label{fig:crs_rci_main}
    \end{subfigure}
    \caption{\small The robust value function and robust cost value function of our algorithm and other baselines on the CRS environment.}
    \label{fig:crs_ib_cf_vf_main}
\end{figure*}


\paragraph{Cost-based Garnet:} The Garnet problem \cite{ganguly_neurips} is a standard benchmark in reinforcement learning. In the cost-based setting, the objective is to maximize cumulative reward subject to a constraint on accumulated cost (see Appendix~\ref{app:experiments}). A defining feature of Garnet environments is their sparse transition structure, where each state–action pair connects to only a limited subset of successor states. For both these environments, we perturb the environment by $\rho=0.05$. We provide results for different values of $\rho$ in Appendix~\ref{app:experiments}.

\begin{figure*}[h!]
    \centering
    \begin{subfigure}[h]{0.5\textwidth}
        \centering
        \includegraphics[clip=2in 2in 0in 2in, height=2.1in]{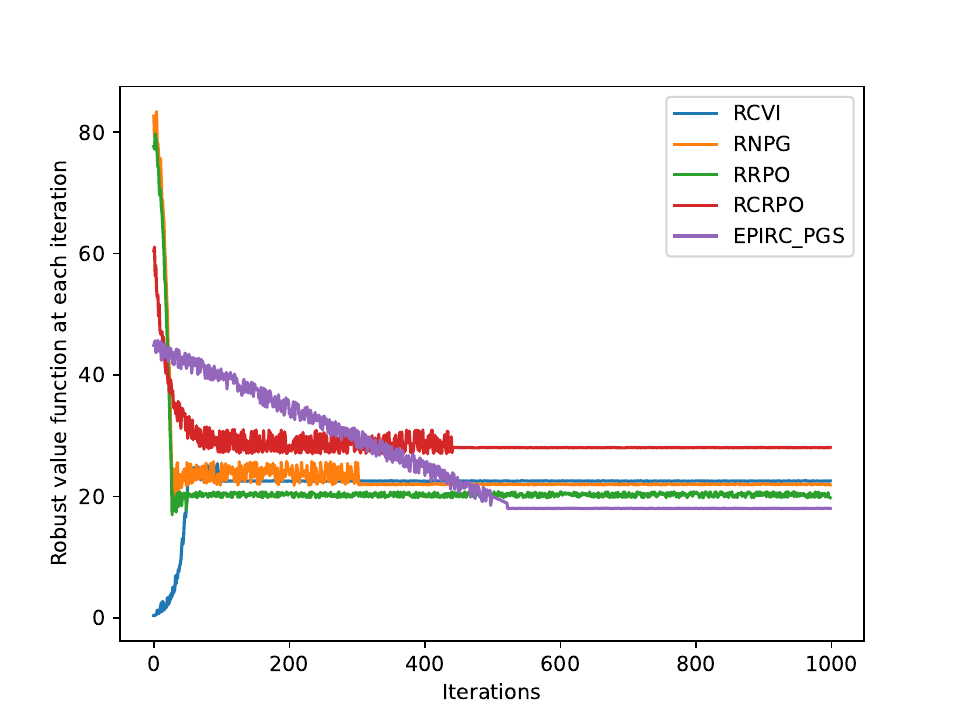}
         \caption{The value function at each iteration step}
        \label{fig:garnet_rvi}
    \end{subfigure}%
    ~ 
    \begin{subfigure}[h]{0.5\textwidth}
        \centering
        \includegraphics[clip=2in 2in 0in 2in, height=2.1in]{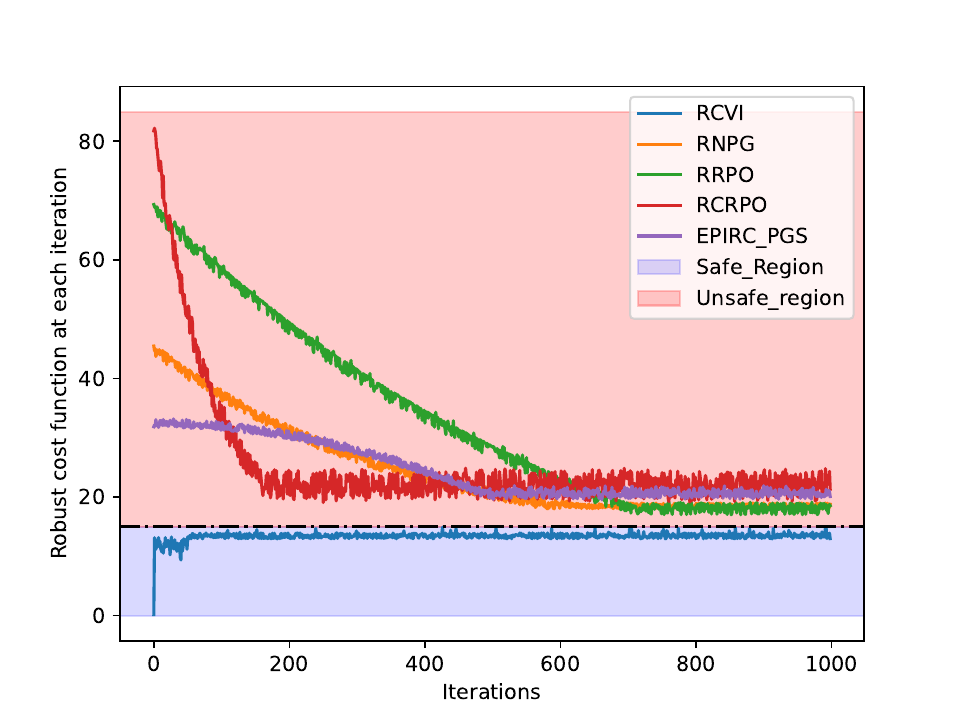}
        \caption{The cost function at each iteration step}
        \label{fig:garnet_rci}
    \end{subfigure}
    \caption{\small Robust value function and cost function update at each iteration on Garnet environment.}
    \label{fig:gar_ib_cf_vf_main}
\end{figure*}

\textbf{Results}: The performance of Algorithm~\ref{algo:rs-cmdp} on CRS is shown in Figure~\ref{fig:crs_ib_cf_vf_main}. As illustrated in Figure~\ref{fig:crs_rvi_main}, the RCVI quickly converges to the maximum value achievable under the budget constraint. While some baselines attain  similar or higher reward, their induced policies {\em violate the safety constraint}, as evidenced in Figure~\ref{fig:crs_rci_main}. This highlights the effectiveness of our approach: unlike competing methods that optimize only over the Markovian policy space, Algorithm~\ref{algo:rs-cmdp} achieves high reward while maintaining feasibility throughout training. Note that our algorithm finds a feasible policy from the initial step, while the Markovian policies in the original state find it difficult to find a feasible policy. 

To further validate these observations in a different setting, we evaluate our method on the cost-based Garnet environment (see Appendix~\ref{app:experiments} for details). The same trend persists: as shown in Figure~\ref{fig:gar_ib_cf_vf_main}, Algorithm~\ref{algo:rs-cmdp} consistently satisfies the constraint, unlike the other baselines. In Figure~\ref{fig:garnet_rvi}, our algorithm  converges rapidly and at a faster rate compared to other approaches. Although the worst-case reward achieved is lower than most baselines (except \texttt{EPIRC-PGS}), it remains the only algorithm that remains feasible throughout training, underscoring the reward-safety tradeoff handled by our approach (Figure~\ref{fig:garnet_rci}). 

Overall, these experiments demonstrate that incorporating explicit budget information significantly improves constraint satisfaction—our method is the only one that consistently learns feasible policies—while still achieving competitive rewards. Additional experiments in Appendix~\ref{app:experiments} further validate the effectiveness of RCVI across different settings. In particular, RCVI exhibits faster convergence while maintaining feasibility with comparable performance. We emphasize that in Figures~\ref{fig:crs_ib_cf_vf_main}, and \ref{fig:gar_ib_cf_vf_main}, thresholds are chosen stringently to ensure that Markovian policies may not find feasible policies, while the augmented policy was able to find. We provide experimental results with different values of $b$ in Appendix. 

\textbf{Test Time Performance}: We evalaute our algorithm on the perturbed environment by taking the best performing policy in the training (see Appendix~\ref{testing_CRS} and \ref{testing_garnet}). We observe that only RCVI satisfies the safety constraint in both environments (see Tables~\ref{tab:testing_vf_CRS}, \ref{tab:testing_cf_CRS}, \ref{tab:rvf_garnet}, and \ref{tab:rcf_garnet}),  with comparable returns with respect to other baselines showing the efficacy of our proposed approach. 

\section{Conclusion and Future Works}
We considered an episodic RCMDP framework. We showed that, unlike in
unconstrained robust MDPs and non-robust CMDPs, policies that are Markov in the
original state can be strictly sub-optimal. We also showed that static and
dynamic adversaries can evaluate the same policy differently. Augmenting the
state with the remaining utility budget and modeling the adversary dynamically
restores a rectangular structure. We consider a per-step feasibility problem with the dynamic adversary. We show that our algorithm  RCVI, needs
$\tilde{\cO}(|S||A|H^5/\epsilon^2)$ samples for TV, $\chi^2$, and KL uncertainty
sets to achieve at most $\epsilon$-violation while being $\epsilon$-close to the optimal robust value for the per-step feasible policy. To our knowledge, this is the first sample complexity guarantee for
RCMDPs that requires no policy optimization oracle. 

Competing with the optimum
of the expected-constraint problem~(\ref{eq:aug_prob}) itself, rather than with
the per-step benchmark, remains an open question. Empirical results support
the approach.
  We consider one constraint. We can address multiple constraints by augmenting the state-space for each individual constraint. The augmented state-space grows with the number of constraints. Thus, extending this work to the multiple constraints while limiting the state-space dimension is an important future research direction. We also consider finite state-space. Extending this work to the function approximation setting also constitutes an important future research direction.  

\bibliographystyle{IEEEtran}
\bibliography{neurips_rcmdp,ref,main,aistats_new,bib_latest}

\appendix
\subsection{Experiments}
\label{app:experiments}
The experiments were performed on two very popular benchmarks in RL\footnote{Lower threshold experiments ($b=4.5$ for RCVI in CRS) were done in colab without use of hardware accelerators such as GPU or TPU for higher threshold $b=30$ for CRS and for Garnet are done in a local machine, having 4GB RAM and 1 TB of hard disk, no GPU support is necessary.} (i) Constrained River-swim (CRS) and (ii) Cost-based Garnet environment. For both the experiments, the f-divergence metric is assumed to be KL-divergence. Although the algorithm is not limited to KL-divergence and can be extended to other f-divergence measures such as TV-distance, $\chi^{2}$-distribution etc.  There are two important reasons for choosing the KL-divergence over the other f-divergence metrics. First the existence of a closed form evaluation method that makes the robust policy evaluation having a known nominal model ($P_{0}$), simple (see Appendix \ref{kl_divergence}). Secondly, the KL-divergence is very stable with minimum influence of the changing hyperparameters. The uncertainty is introduced in both environments by adding noise to the nominal model transition probabilities and then normalizing them to make a probability distribution during training time and by reversing the probability distribution in the testing or inference time.

\subsection{Constrained River-swim}
The constrained River-swim (CRS) is an important benchmark environment studied in optimization theory and control. We briefly introduce the objective of the Constrained River-swim environment, followed by the results obtained upon training our algorithm (Algorithm \ref{algo:rs-cmdp}) on CRS.

\subsubsection{Environment Description}
The environment consists of $6$ states, each representing a landmass.  A swimmer starts in one of these states according to a random initial distribution.  At any state, the swimmer can choose between two actions: \emph{swim left} ($a_{0}$) or \emph{swim right} ($a_{1}$).  Each action leads to a probabilistic transition to the next state, governed by a transition distribution.  

The swimmer receives rewards only at the extreme states (in this case, $s_{0}$ and $s_{5}$).  However, the environment also introduces risks:  
\begin{itemize}
    \item A \textbf{river current} always pushes against the swimmer’s chosen direction, making movement more uncertain.  
    \item \textbf{Harmful creatures} inhabit the landmasses, causing injury to the swimmer. The number of these creatures increases as we move to higher-indexed states.  
\end{itemize}

Thus, $s_{0}$ is the safest state with minimal reward and minimal cost, while states with larger indices carry linearly increasing costs due to greater safety hazards.  The swimmer’s objective is to learn an optimal policy that balances reward and safety, despite uncertainties in the transition dynamics, given only a nominal transition probability distribution $P_{0}$.

\subsubsection{Implementation details}
Constrained River-swim environment consists of $6$ states so, let us denote the six states as $\mathcal{S} = \{s_{0},s_{1},s_{2},\ldots,s_{5}\}$ and $2$ actions denoted as $\mathcal{A} =\{a_{0},a_{1} \}$ where $\mathcal{S}$ and $\mathcal{A}$ denotes the state space and action space respectively. The above information is important for the underlying model of the environment that we considered (as shown in tables \ref{tab:rs} and \ref{tab:rs_r_c}).

\begin{table}[h]
    \centering
    \begin{tabular}{ccc}
    \toprule
        \textbf{State} & \textbf{Action} & \textbf{Probability for next state}  \\
    \midrule
        $s_{0}$ & $a_{0}$ & $s_{0}$:0.9, $s_{1}$:0.1 \\
    \midrule
        $s_{i}, \; i \in \{1,2,\ldots, 5\}$ & $a_{0}$ & $s_{i}$:0.6, $s_{i-1}$:0.3, $s_{i+1}$:0.1 \\
    \midrule
        $s_{i}, \; i \in \{0,1,\ldots, 4\}$ & $a_{1}$ & $s_{i}$:0.6, $s_{i-1}$:0.1, $s_{i+1}$:0.3 \\
    \midrule
        $s_{5}$ & $a_{1}$ & $s_{5}$:0.9, $s_{4}$:0.1 \\
    \bottomrule
    \end{tabular}
    \caption{Transition probabilities of the RiverSwim environment.}
    \label{tab:rs}
\end{table}

\begin{table}[h]
    \centering
    \begin{tabular}{ccc}
    \toprule
       \textbf{State}  &\textbf{Reward}&\textbf{Constraint cost}  \\
       \midrule
         $s_{0}$&0.001&0.02\\
         \midrule
         $s_{1}$&0&0.035\\
         \midrule
         $s_{2}$&0&0\\
         \midrule
         $s_{3}$&0&0.01\\
         \midrule
         $s_{4}$&0.1&0.08\\
         \midrule
         $s_{5}$&1&0.9\\
         \bottomrule
    \end{tabular}
    \caption{The reward and constraint cost received at each state}
    \label{tab:rs_r_c}
\end{table}

Now, that we have discussed about the model of the underlying CRS environment, let us now list the hyperparameters used (see Table \ref{tab:hp_CRS}).

\begin{table}[h]
    \centering
    \begin{tabular}{ccc}
         &\textbf{Hyperparameters}&\textbf{Value}  \\
         \toprule
        \multirow{3}{5em}{Environmental parameters} &$B$ (Budget)&4.5\\
         &$|\mathcal{S}|$&6\\
         &$|\mathcal{A}|$&2\\
         \midrule
         \multirow{5}{5em}{Algorithm \ref{algo:rs-cmdp} variables}&bins&10\\
         &$\rho$(f-divergence tolerence)&0.05\\
         &$H$ (horizon length)&1000\\
         &$N$ (sample size)&1000\\
         \bottomrule
    \end{tabular}
    \caption{Hyperparameter list for running Algorithm \ref{algo:rs-cmdp} on CRS}
    \label{tab:hp_CRS}
\end{table}

\subsubsection{Additional Results and discussion}
In addition to the experiments reported in Figure~\ref{fig:crs_ib_cf_vf_main}, we conduct several supplementary experiments to further illustrate the limitations of the Markovian policy setup and the advantages of the augmented formulation.

The choice of budget threshold $b=4.5$ may appear unconventional and could raise questions regarding its selection. Our motivation, however, is deliberate: we choose $b$ so that none of the competing algorithms can identify a feasible policy under this budget. This helps expose a fundamental limitation of the Markovian setup, namely its inability to recover feasible policies when all candidate policies lie in the unsafe region, a limitation that does not arise in the augmented setting.

\paragraph{Test Time  results} 
\label{testing_CRS}
Since our training procedure incorporates robustness, it is important to verify whether these benefits persist at inference (test) time under perturbed dynamics. To evaluate this, we deliberately perturb the transition probabilities during testing. Specifically, if the agent is at state $s_t=x$, it can transition to one of $s_{t+1}=\{x-1,x,x+1\}$ with probabilities $p_1,p_2,p_3$, respectively. During testing, with probability $\epsilon$, we randomly alter the transition ordering; for example, the transition probabilities may change from $(p_1,p_2,p_3)$ to $(p_3,p_2,p_1)$. For our test-time, we set $\epsilon=0.05$.

 We initialize from random states, sample actions from the learned Q-table using the same cost discretization, and run a reduced evaluation over $100$ independent episodes. We then report the average robust value function and robust cost function, summarized in Tables~\ref{tab:testing_vf_CRS} and \ref{tab:testing_cf_CRS}.

\begin{table}[h!]
    \centering
    \begin{tabular}{|c|c|c|c|c|}
    \hline
        \textbf{RCVI} &RNPG&RRPO&EPIRC-PGS&RCRPO  \\
        \hline
        \textbf{5.5044} &6.0456&5.00031& 14.113&12.7812\\
        \hline
    \end{tabular}
    \caption{Average robust value functions of 100 steps for all the algorithms}
    \label{tab:testing_vf_CRS}
\end{table}

\begin{table}[h!]
    \centering
    \begin{tabular}{|c|c|c|c|c|}
    \hline
        \textbf{RCVI} &RNPG&RRPO&EPIRC-PGS&RCRPO   \\
        \hline
         \textbf{4.4081}&8.2304&6.00214&9.211&11.5488 \\
         \hline
         
    \end{tabular}
    \caption{Average robust cost functions of 100 steps for all the algorithms,  $b=4.5$.}
    \label{tab:testing_cf_CRS}
\end{table}
\subsubsection{Training with an increased safety threshold value}
\paragraph{a) Training with a safety threshold $b=7$} In addition to the above setup, we run our algorithm with a higher safety threshold value to underscore the performance of RCVI. We first consider $b=7$ and $\rho=0.1$.  Since $b$ is high, now the Markovian policies in the original state-space is feasible, yet we observe that our proposed approach (RCVI) converges faster compared to the other baselines (Figure~\ref{fig:one_safe_pol_crs}).   From Figure \ref{one_safe_pol_crs_vf}, we note that all the algorithms have comparable value functions, and Figure \ref{one_safe_pol_crs_cf} tells us that all the algorithms are in the safe zone. 

\begin{figure*}[h!]
    \centering
    \begin{subfigure}[h]{0.5\textwidth}
        \centering
        \includegraphics[height=2.2in]{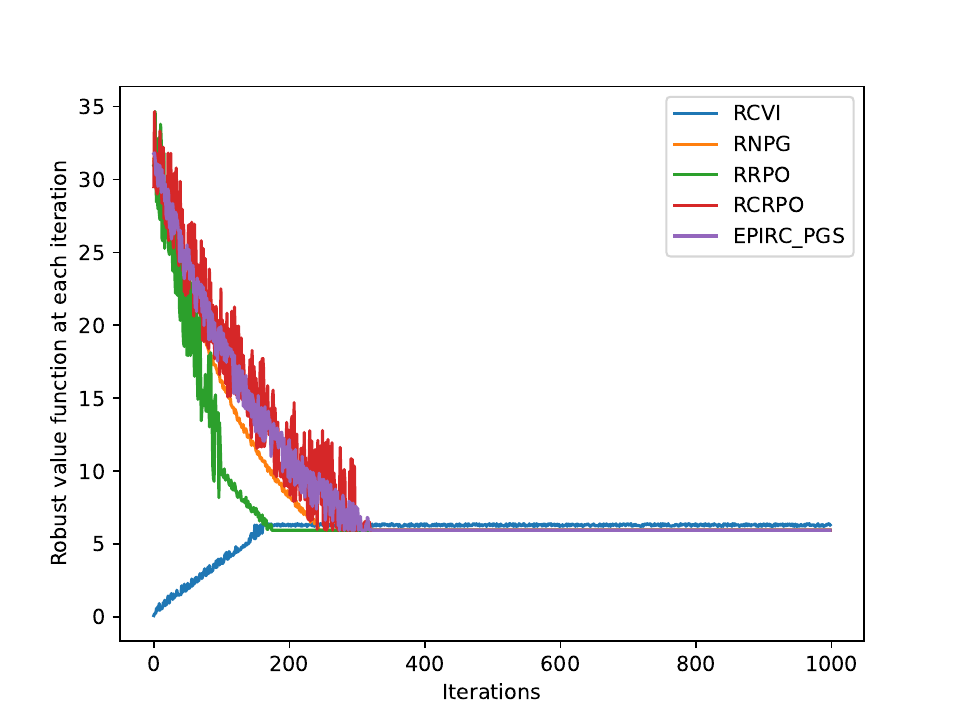}
         \caption{The value function at each iteration step}
        \label{one_safe_pol_crs_vf}
    \end{subfigure}%
    ~ 
    \begin{subfigure}[h]{0.5\textwidth}
        \centering
        \includegraphics[height=2.2in]{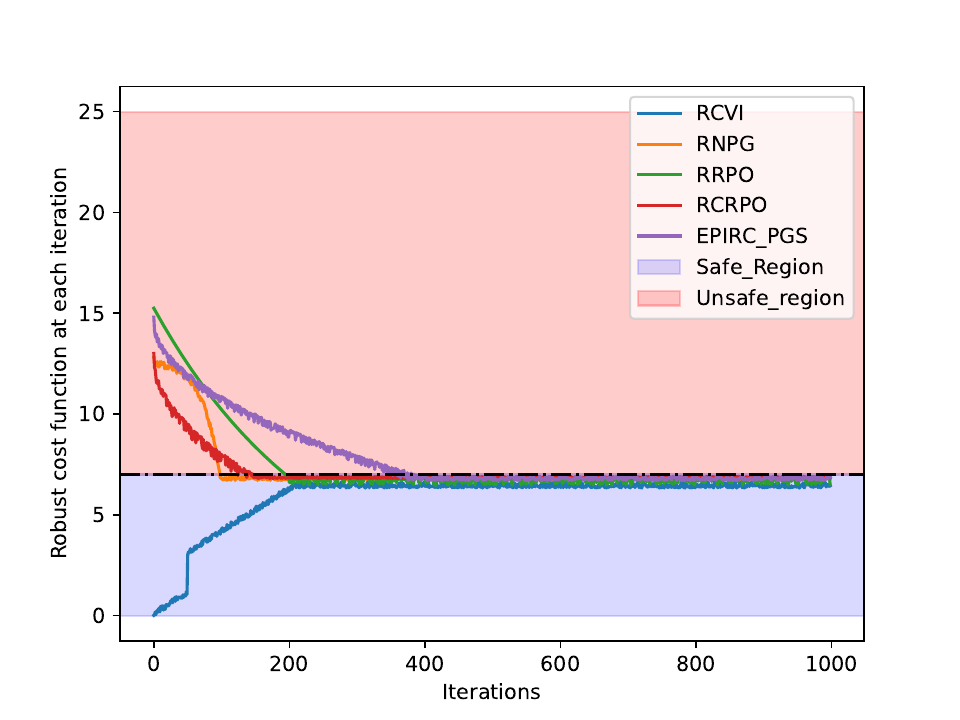}
        \caption{The cost function at each iteration step}
        \label{one_safe_pol_crs_cf}
    \end{subfigure}
    \caption{\small Robust value function and cost function update at each iteration on Constrained River-swim(CRS environment) $b=7$.}
    \label{fig:one_safe_pol_crs}
\end{figure*}

\paragraph{b) Training with a safety threshold $b=30$}
We consider a higher value of threshold $b=30$ and $\rho=0.1$. . Here, the Markovian policy in the original state-space is feasible, and all the other baselines converge to the optimal policy as well. In this setup too, we find that the convergence of RCVI is faster compared to other state-of-the-art approaches (Figure~\ref{fig:crs_ib_cf_vf}. RCVI also achieves a similar reward while maintaining feasibility throughout the training. {\em This set of studies  tell us that RCVI can  not only find a feasible policy but also the highest-returning feasible policy across different values of $b$}. Further, RCVI converges  faster compared to  all the state-of-the-art approaches.  
\begin{figure*}[h!]
    \centering
    \begin{subfigure}[h]{0.5\textwidth}
        \centering
        \includegraphics[height=2.3in]{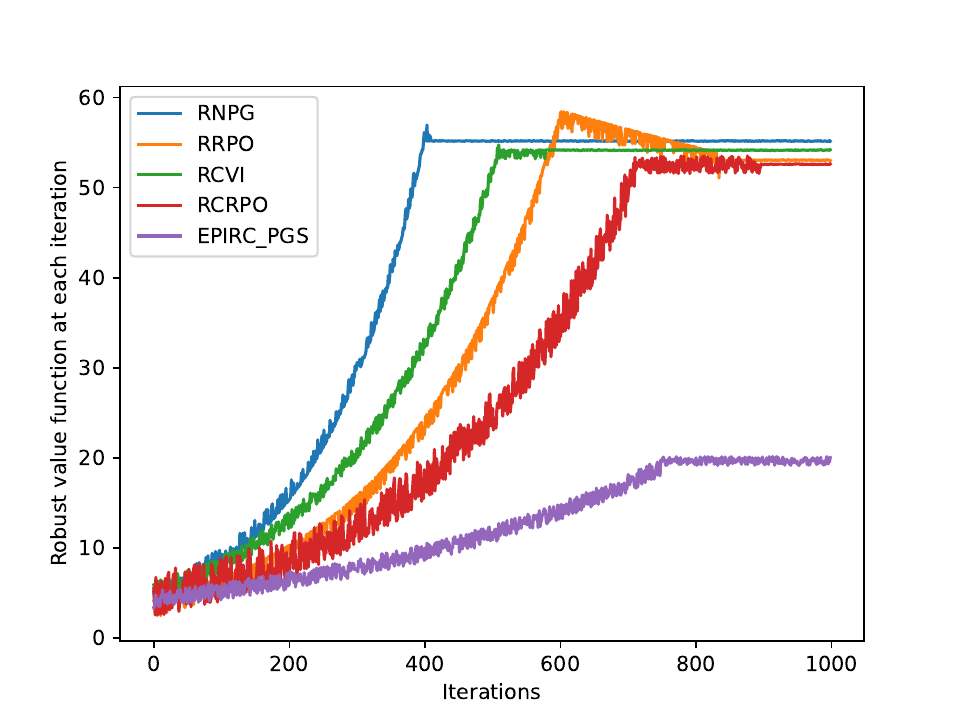}
         \caption{The value function at each iteration step}
        \label{fig:crs_ib_vf}
    \end{subfigure}%
    ~ 
    \begin{subfigure}[h]{0.5\textwidth}
        \centering
        \includegraphics[height=2.3in]{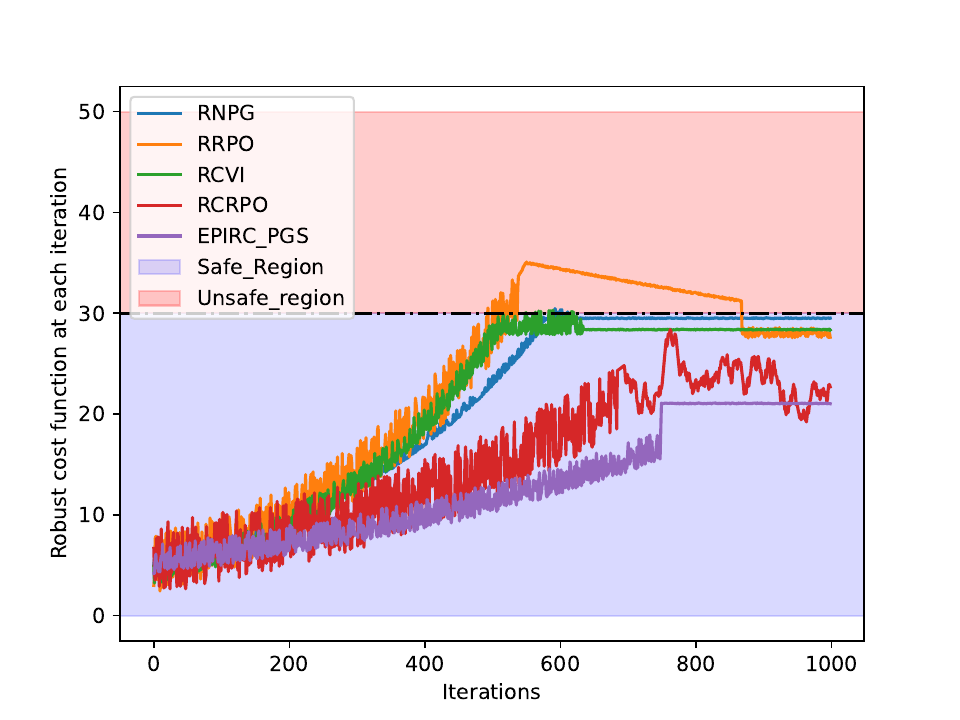}
        \caption{The cost function at each iteration step}
        \label{fig:crs_ib_cf}
    \end{subfigure}
    \caption{\small Robust value function and cost function update at each iteration on Constrained River-swim (CRS) environment for $b=30$.}
    \label{fig:crs_ib_cf_vf}
\end{figure*}
\subsubsection{Ablation study on uncertainty measure $\rho$}
In this section, we perform an ablation study on different values of $\rho$ for the CRS environment. The results of which can be found in Figure \ref{fig:CRS_app_list} where we compare with the baselines.

\begin{figure*}[]
    \centering
    \begin{subfigure}[]{0.5\textwidth}
        \centering
        \includegraphics[height=2.3in]{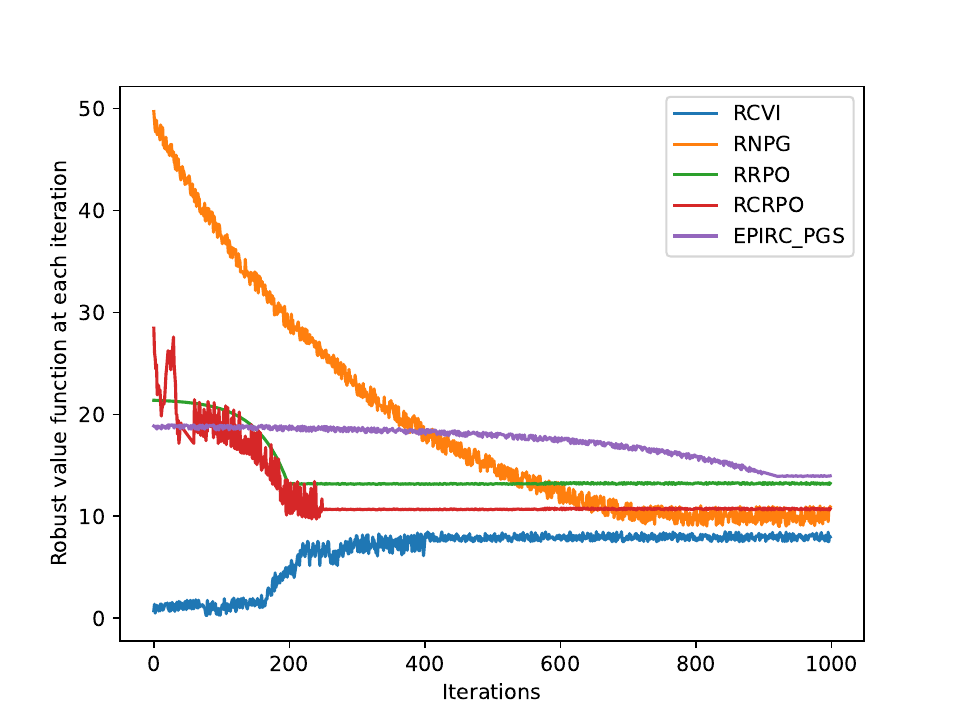}
         \caption{The value function at each iteration step ($\rho=0.01$)}
        \label{fig:CRS_vf_app_0.01}
    \end{subfigure}%
    ~ 
    \begin{subfigure}[]{0.5\textwidth}
        \centering
        \includegraphics[height=2.3in]{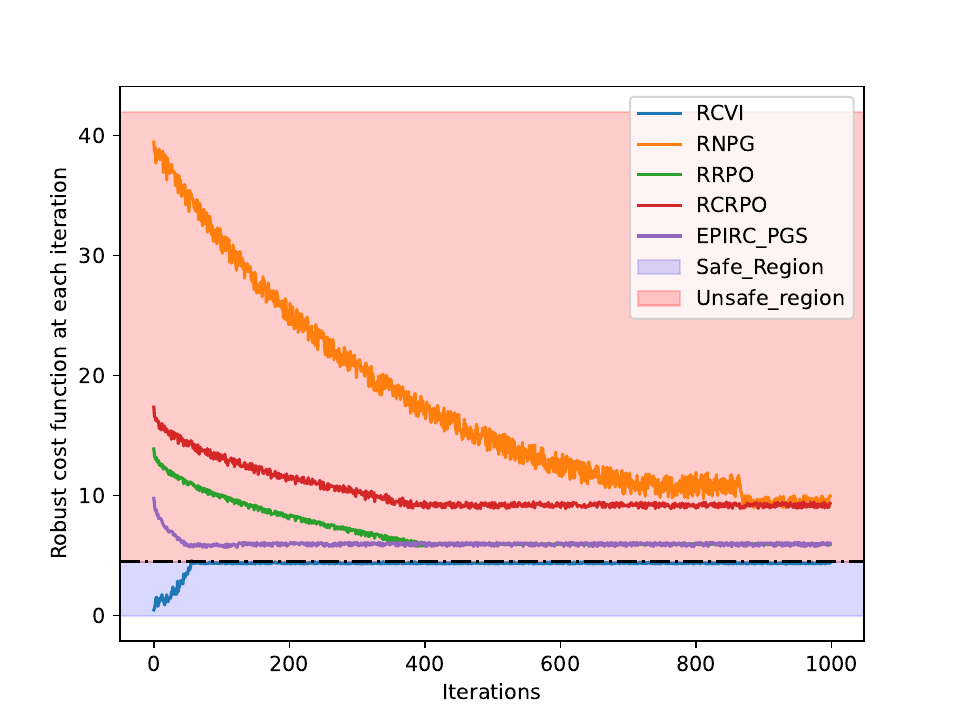}
        \caption{The cost function at each iteration step ($\rho=0.01$)}
        \label{fig:CRS_cf_app_0.01}
    \end{subfigure}%
    \\
    \begin{subfigure}[]{0.5\textwidth}
        \centering
        \includegraphics[height=2.3in]{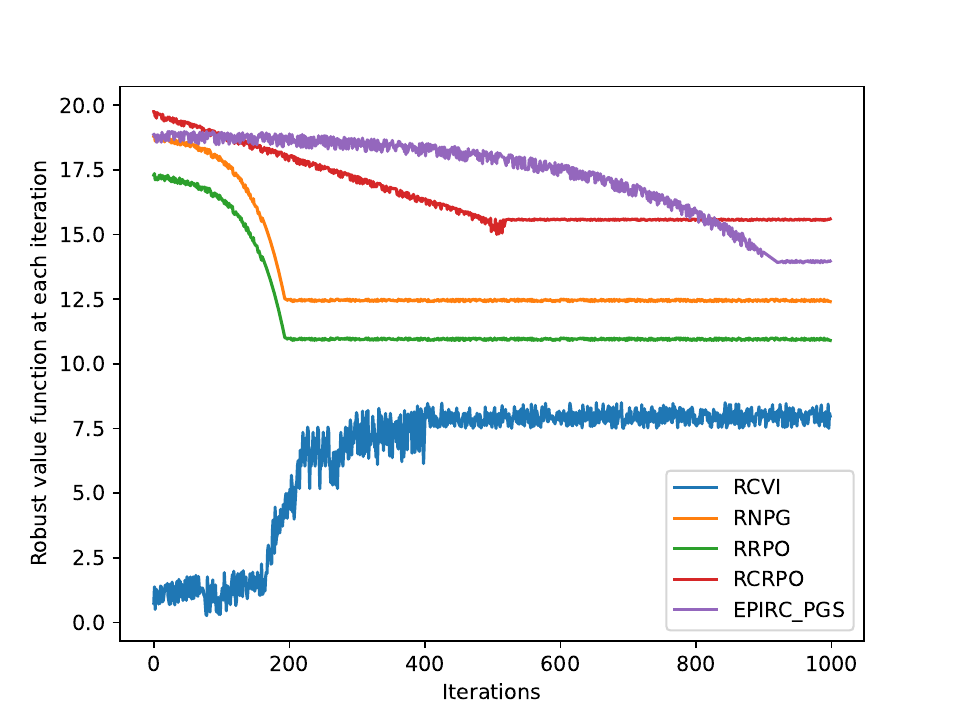}
        \caption{The value function at each iteration step($\rho=0.05$)}
        \label{fig:CRS_vf_app_0.05}
    \end{subfigure}%
    ~
    \begin{subfigure}[]{0.5\textwidth}
        \centering
        \includegraphics[height=2.3in]{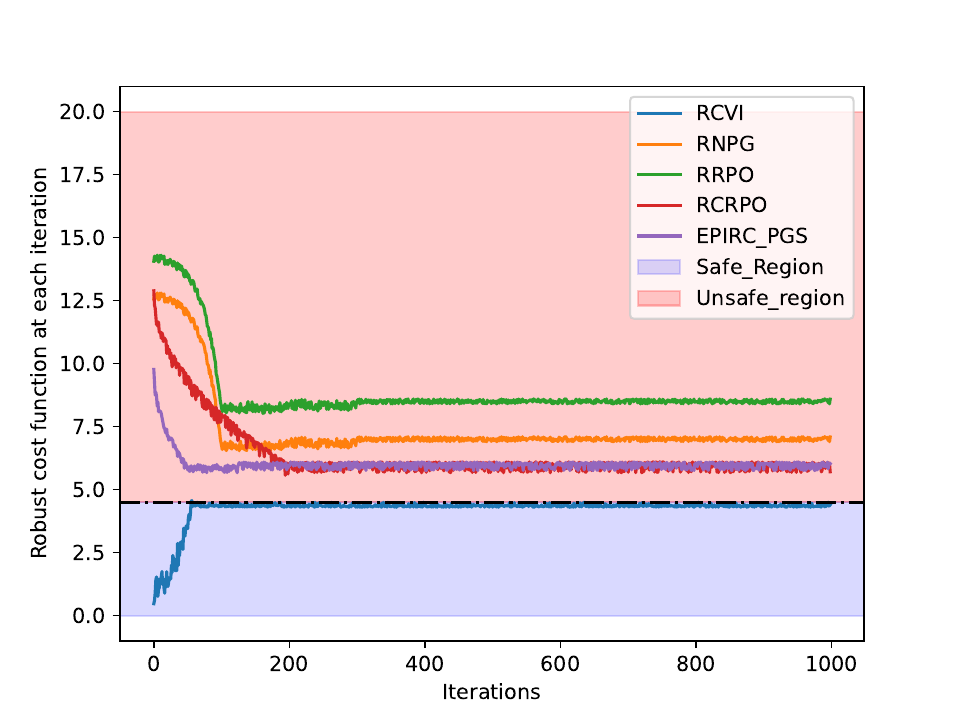}
        \caption{The cost function at each iteration step($\rho=0.05$)}
        \label{fig:CRS_cf_app_0.05}
    \end{subfigure}%
    \\
    \begin{subfigure}[h]{0.5\textwidth}
        \centering
        \includegraphics[height=2.3in]{images/Ablation_CRS_rho_0.1_vf.pdf}
         \caption{The value function at each iteration step($\rho=0.1$)}
        \label{fig:CRS_vf_app_0.1}
    \end{subfigure}%
    ~ 
    \begin{subfigure}[h]{0.5\textwidth}
        \centering
        \includegraphics[height=2.3in]{images/Ablation_CRS_rho_0.1_cf.pdf}
        \caption{The cost function at each iteration step($\rho=0.1$)}
        \label{fig:CRS_cf_app_0.1}
    \end{subfigure}
    \caption{The x-axis denotes the Iteration number and the y-axis denotes the worst case reward value function at that iteration (left figure) and the worst case cost function at that iteration (right figure) for various algorithms on the CRS environment for different values of $\rho$.}
    \label{fig:CRS_app_list}
\end{figure*}

The first row of plots (Figures~\ref{fig:CRS_vf_app_0.01} and~\ref{fig:CRS_cf_app_0.01}) corresponds to $\rho = 0.01$, the following row of plots correspond to $\rho = 0.05$ (Figures~\ref{fig:CRS_vf_app_0.05} and~\ref{fig:CRS_cf_app_0.05}) and the final row (Figures~\ref{fig:CRS_vf_app_0.1} and~\ref{fig:CRS_cf_app_0.1}) of plots correspond to $\rho=0.1$. The baseline was kept at $4.5$ for all the experiments. In all cases, baseline algorithms such as RNPG, RRPO, RCRPO, and EPIRC-PGS operate in the unsafe region, requiring large iterations to converge. In contrast, our RCVI algorithm consistently remains within the safe boundary,  and converges to the optimal policy quickly.

\subsection{Cost-based Garnet problem}
The Garnet problem is another fundamental benchmark problem used in Control theory and Reinforcement Learning to test the effectiveness of an algorithm. In the next subsection we briefly discuss about the environment with the results in the following subsection
\subsubsection{Environment Description}
The \emph{Garnet environment} is a widely used MDP benchmark designed for evaluating RL algorithms under controlled conditions. 
It is defined by a fixed number of states \(nS\) and actions \(nA\), with transition probabilities, rewards, and (in constrained RL) utility functions sampled from prescribed distributions. 
A key feature of the Garnet setup is that the transition dynamics are typically \emph{sparse}, meaning that each state--action pair leads only to a limited subset of possible successor states rather than all states.  

Formally, the environment is specified by a transition probability kernel \(P(s' \mid s, a)\), a reward function \(R(s, a)\), and, when applicable, a utility function \(U(s, a)\). 
These quantities are often sampled from normal distributions:
\begin{align*}
P(s' \mid s, a) \sim \mathcal{N}(\mu_a, \sigma_a),\\ 
R(s, a) \sim \mathcal{N}(\mu_b, \sigma_b),\\ 
 U(s, a) \sim \mathcal{N}(\mu_c, \sigma_c),
\end{align*}
where the means \(\mu_a, \mu_b, \mu_c\) are themselves drawn from a uniform distribution, i.e., 
\(\text{Unif}(0,10)\).  

Since \(P(s' \mid s, a)\) must define a valid probability distribution (each row summing to one), 
The sampled values are exponentiated and normalized via a softmax transformation:
\[
p^{0}(s' \mid s, a) 
= \frac{\exp(P(s' \mid s, a))}{\sum_{s''} \exp(P(s'' \mid s, a))}.
\]

However, our cost-based Garnet environment is a slightly changed version of the aforementioned environment. In this cost based setting, we assumed a cost function $C(s, a) \sim \mathcal{N}(\mu_{c},\sigma_{c})$ instead of the utility function $U(s, a)$ and the new objective of the cost based Garnet environment is to maximize the long term objective function while keeping the long term cost function denoted as $V_{P,c}^{\pi}=\sum_{t=1}^{T}\mathbb{E}_{P}\left[ C(s_{t},\pi_{t}(s_{t}))\right] \leq b$ below a certain threshold value $b$.

\subsubsection{Implementation details}
The Garnet environment had $10$ states denoted as $s_{0} \ldots s_{9}$ and $5$ distinct actions denoted as $a_{0} \ldots a_{4}$. The hyperparameter list is as given below in table \ref{tab:hp_garnet}

\begin{table}[h]
    \centering
    \begin{tabular}{ccc}
    \toprule
         & \textbf{Hyperparameters}&\textbf{Value} \\
         \midrule
         \multirow{5}{6em}{Environmental parameters}&B&15\\
         &$|\mathcal{S}|$&10\\
         &$|\mathcal{A}|$&5\\
         &$\mu_{a},\mu_{b},\mu_{c}$&$Unif(0,10)$,$Unif(0,10)$, $Unif(0,10)$\\
         &$\sigma_{a},\sigma_{b},\sigma_{c}$&$Unif(0,10)$,$Unif(0,10)$, $Unif(0,10)$\\
         \midrule
         \multirow{5}{6em}{Algorithm \ref{algo:rs-cmdp} variables}&bins&$20$\\
         &$\rho$ (f-divergence tolerance)&$\{0.05,0,01,0.1\}$\\

         &$H$&$1000$\\
         &$N$&$1000$\\
         \bottomrule
    \end{tabular}
    \caption{Hyperparameters used for running Algorithm \ref{algo:rs-cmdp} on Garnet environment}
    \label{tab:hp_garnet}
\end{table}

\subsubsection{Additional Results and discussion}
Similar to the CRS environment, we conduct additional experiments to support the advantages of our approach. The choice of budget for most of the pilot studies, the testing time study, and the ablation study was 15.

\paragraph{Testing Time (Inference time) results}
\label{testing_garnet}
During test time, we run $100$ independent episodes where the initial state is chosen randomly. The average robust value function and average robust cost function are listed in tables \ref{tab:rvf_garnet} and \ref{tab:rcf_garnet}. Only the RCVI achieves a feasible policy while also achieving a good reward. 

\begin{table}[h!]
    \centering
    \begin{tabular}{|c|c|c|c|c|}
    \hline
        \textbf{RCVI} &RNPG&RRPO&EPIRC-PGS&RCRPO  \\
        \hline
        \textbf{21.0125} &22.9541&20.1546& 20.5441&31.1546\\
        \hline
    \end{tabular}
    \caption{\small Average robust value functions of 100 steps for all the algorithms (Garnet environment)}
    \label{tab:rvf_garnet}
\end{table}

\begin{table}[h!]
    \centering
    \begin{tabular}{|c|c|c|c|c|}
    \hline
        \textbf{RCVI} &RNPG&RRPO&EPIRC-PGS&RCRPO   \\
        \hline
         \textbf{14.9562}&23.5541&19.8456&22.7741&24.1179 \\
         \hline
         
    \end{tabular}
    \caption{\small Average robust cost functions of 100 steps for all the algorithms (Garnet environment), $b=15$.}
    \label{tab:rcf_garnet}
\end{table}
\subsubsection{Training with an increased safety threshold value}
\paragraph{Training with a safety threshold $b=25$:}  We begin by studying a slightly higher baseline $b=25$. The findings from this study are compiled in Figures \ref{one_safe_pol_garnet_vf} and \ref{one_safe_pol_garnet_cf}. 
\begin{figure*}[h!]
    \centering
    \begin{subfigure}[h]{0.5\textwidth}
        \centering
        \includegraphics[height=2.2in]{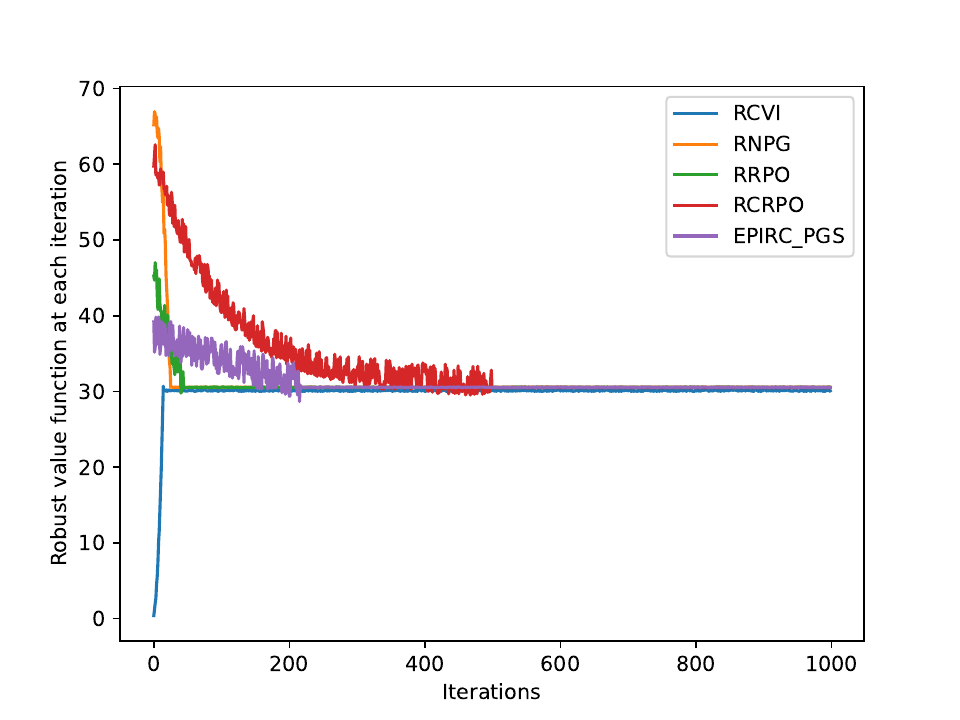}
         \caption{The value function at each iteration step}
        \label{one_safe_pol_garnet_vf}
    \end{subfigure}%
    ~ 
    \begin{subfigure}[h]{0.5\textwidth}
        \centering
        \includegraphics[height=2.2in]{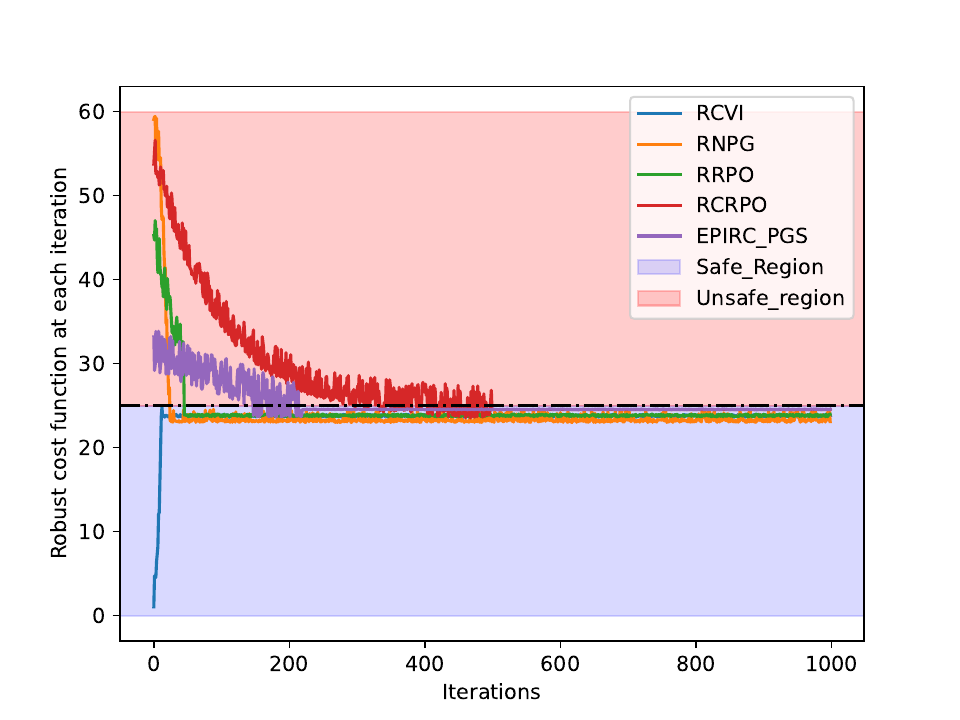}
        \caption{The cost function at each iteration step}
        \label{one_safe_pol_garnet_cf}
    \end{subfigure}
    \caption{The above plot shows when only 1 policy exists in the feasible zone then all other algorithms finds that policy however, our RCVI algorithm is still effective in such situations (Garnet environment) $b=25$.}
    \label{fig:one_safe_pol_garnet}
\end{figure*}
Since $b$ is high, the Markovian policies are also feasible. As observed from Figure \ref{fig:one_safe_pol_garnet}, all the algorithms converge to a close enough policy with  comparable value function and cost functions. However, our proposed approach RCVI coverges the fastest among all the other approaches. 

\paragraph{Training with a safety threshold $b=50$} Since the threshold is high, again the Markovian policies become feasible. From Figure~\ref{fig:gar_ib_cf_vf} note that our proposed approach is the fastest among all the other approaches. The achieved reward is almost similar with the baselines while maintaining feasibility throughout the training. 

\begin{figure*}[h!]
    \centering
    \begin{subfigure}[h]{0.5\textwidth}
        \centering
        \includegraphics[clip=2in 2in 2in 2in,height=2.3in]{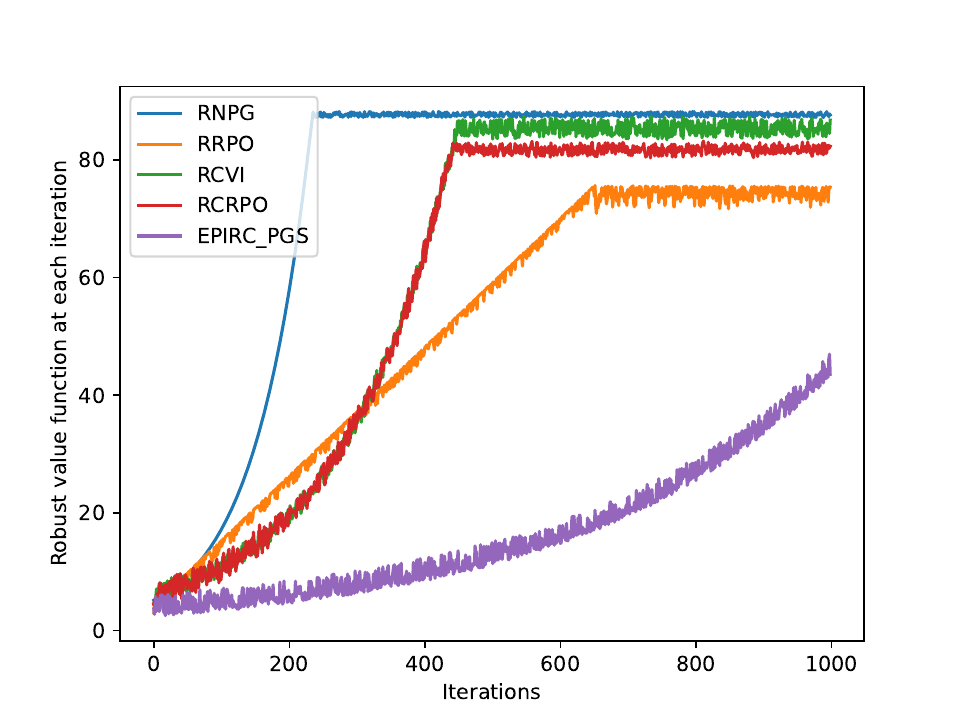}
         \caption{The value function at each iteration step}
        \label{fig:gar_ib_vf}
    \end{subfigure}%
    ~ 
    \begin{subfigure}[h]{0.5\textwidth}
        \centering
        \includegraphics[height=2.3in]{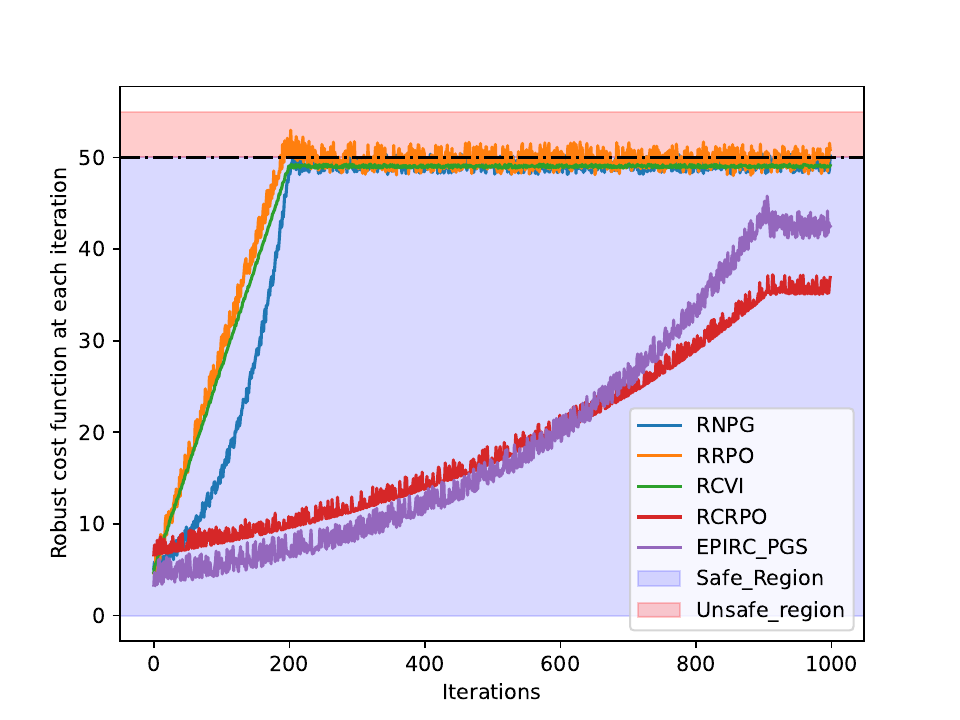}
        \caption{The cost function at each iteration step}
        \label{fig:gar_ib_cf}
    \end{subfigure}
    \caption{\small Robust value function and cost function update at each iteration on Garnet environment (b=50).}
    \label{fig:gar_ib_cf_vf}
\end{figure*}


Evaluating both environments, despite their similar uncertainty estimators and indices, strengthens the robustness and generalizability of the results across different settings.

\subsubsection{Ablation study on uncertainty measure $\rho$}
Similar to the CRS setting, we conduct an ablation study on different values of $\rho$ for the Garnet environment. 
The corresponding results are presented in Figure \ref{fig:Garnet_app_list}.

\begin{figure*}[]
    \centering
    \begin{subfigure}[]{0.5\textwidth}
        \centering
        \includegraphics[height=2.3in]{images/Ablation_Garnet_rho_0.01_vf.pdf}
         \caption{The value function at each iteration step ($\rho=0.01$)}
        \label{app:fig_gar_vf_0.01}
    \end{subfigure}%
    ~ 
    \begin{subfigure}[]{0.5\textwidth}
        \centering
        \includegraphics[height=2.3in]{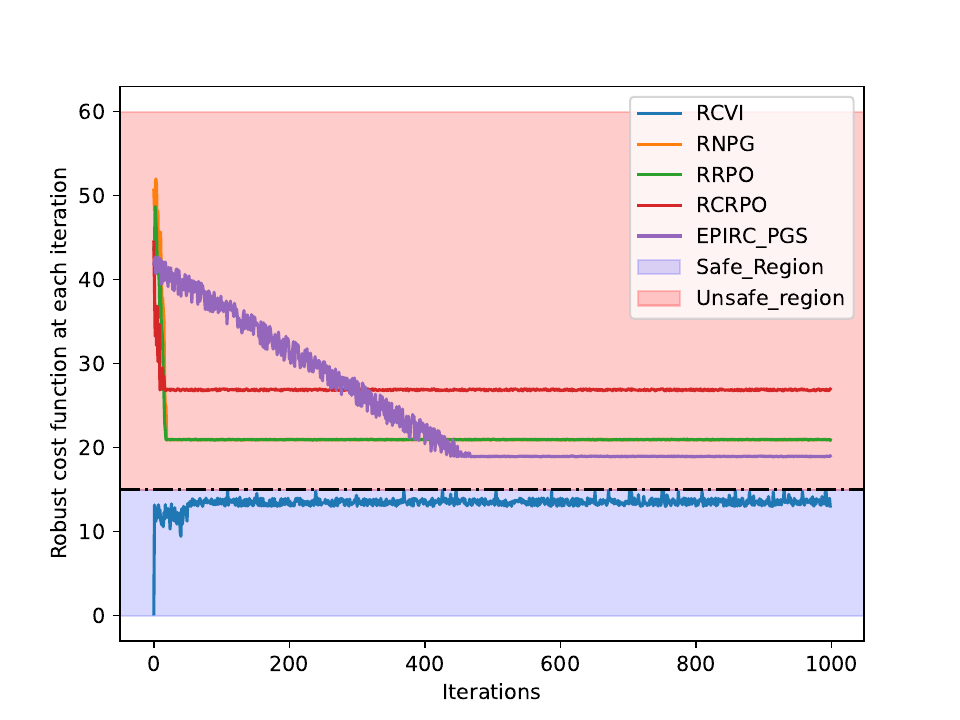}
        \caption{The cost function at each iteration step ($\rho=0.01$)}
        \label{app:fig_gar_cf_0.01}
    \end{subfigure}
    \\
    \begin{subfigure}[]{0.5\textwidth}
        \centering
        \includegraphics[height=2.3in]{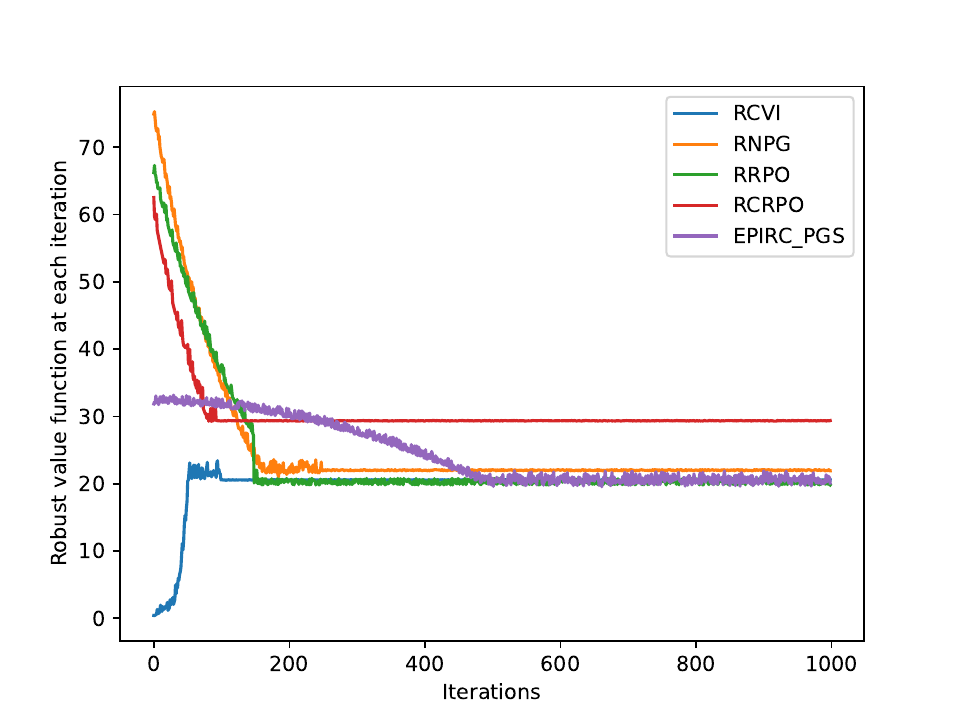}
         \caption{The value function at each iteration step ($\rho=0.05$)}
        \label{app:fig_gar_vf_0.05}
    \end{subfigure}%
    ~ 
    \begin{subfigure}[]{0.5\textwidth}
        \centering
        \includegraphics[height=2.3in]{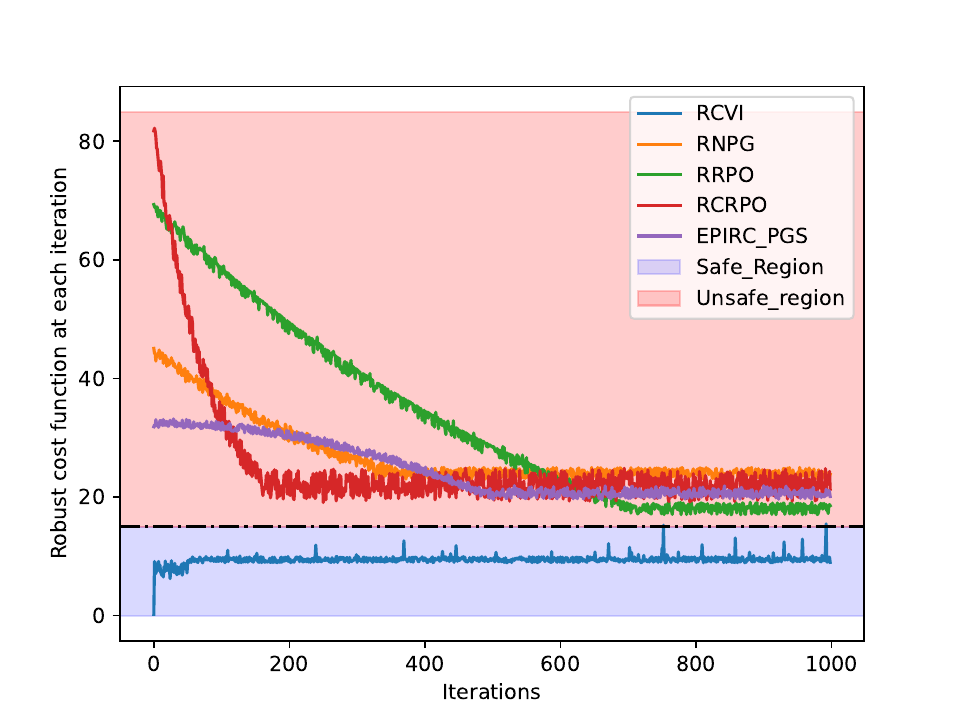}
        \caption{The cost function at each iteration step ($\rho=0.05$)}
        \label{app:fig_gar_cf_0.05}
    \end{subfigure}
    \\
     \begin{subfigure}[]{0.5\textwidth}
        \centering
        \includegraphics[height=2.3in]{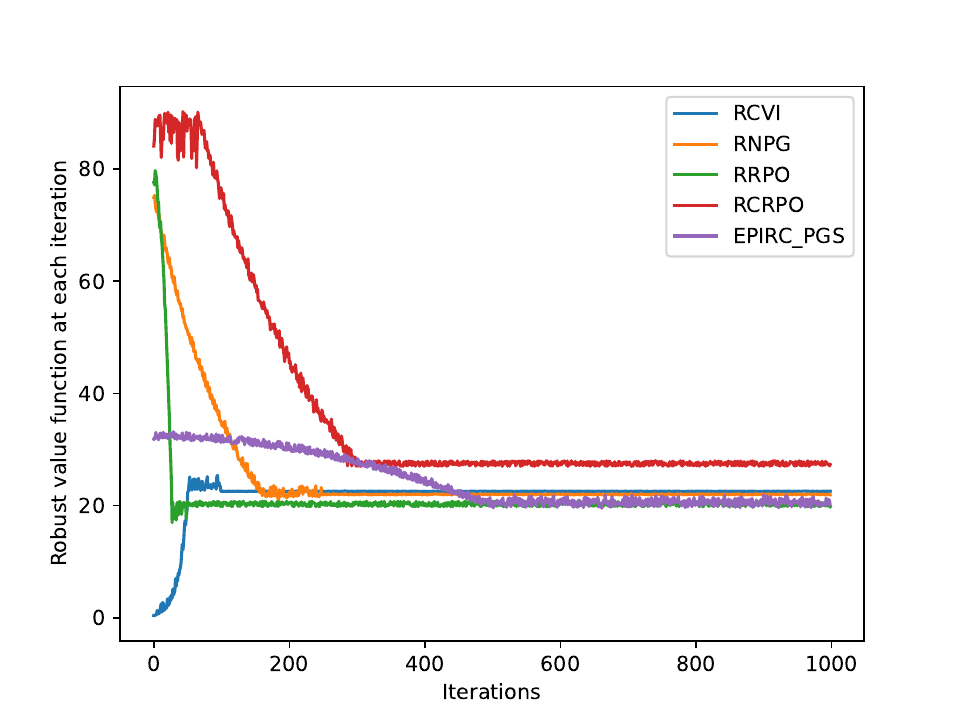}
         \caption{The value function at each iteration step ($\rho=0.1$)}
        \label{app:fig_gar_vf_0.1}
    \end{subfigure}%
    ~ 
    \begin{subfigure}[]{0.5\textwidth}
        \centering
        \includegraphics[height=2.3in]{images/Ablation_Garnet_rho_0.1_cf.pdf}
        \caption{The cost function at each iteration step ($\rho=0.1$)}
        \label{app:fig_gar_cf_0.1}
    \end{subfigure}
    \caption{The worst case reward value function and cost function update at each iteration step where the x-axis denotes the Iteration number and the y-axis denotes the worst case value function at that iteration (left figure) and the worst case cost function at that iteration (right figure)  on the cost based Garnet environment.}
    \label{fig:Garnet_app_list}
\end{figure*}

Figures \ref{fig:Garnet_app_list} shows a detailed comparison when all the algorithms start from the safe zone, we see only RCVI converges to a safe feasible high returning policy due to its structure while others although start safe, after the parameters gets fixed, they converge to a non-feasible policy. This is because, the baseline was deliberately kept such that none of the policies could satisfy. Even for a higher value of $\rho$, RCVI can achieve a feasible policy while the other algorithms cannot.  

\section{Policy evaluation using KL-divergence as the f-divergence measure}\label{kl_divergence}
For KL divergence uncertainty set, the worst case value function can be achieved by the following expression from Proposition 5 in \cite{xu2023improved}
\begin{align}\label{eq:worst_kl}
L_{\cP_{h,x,a}}V
= -\inf_{\lambda\in[0,H/\rho]}\lambda\Big\{\rho+\log\,
\mathbb E_{x'\sim P_h^o(\cdot\mid x,a)}\!\big[e^{\big(-\frac{V(x')}{\lambda}\big)}\big]\Big\}
\end{align}
Then by Lemma 4 in \cite{iyengar2005robust} we have the worst case model for $j=r,g$ as
\begin{align}\label{eq:worst_model_kl}
P_{j,h}^*\propto P_h^0\exp(-V_{j,h+1}(x')/\lambda_j^*)
\end{align}
where $\lambda_j^*$ is the solution of the convex optimization problem in (\ref{eq:worst_kl}) for $V_{j,h}$. Note that here $x'=(s',c-g_h)$, thus, here one only needs to solve it for the state component rather than the budget evolution component. Also, (\ref{eq:worst_kl}) is a convex optimization problem. Yet, solving for all state-action pair is computationally expensive.  We thus consider an update given by \cite{Epigraph} which we adapt for the episodic case recursively starting from state $H+1$. Note that $V_{r,H+1}(\cdot,\cdot)=0$ whereas $V_{g,H+1}(\cdot,c)=-c$. 

In particular, we consider the following update
\begin{align}\label{eq:equivalent}
& Q_{g,h}(s,c,a)=\sum_{s^{\prime}}P_{g,h}^*(s^{\prime}|s,a)V_{g,h+1}(s',c-g_h),\nonumber \\&P_{g,h}^*\propto P_h^0(\cdot|s,a)\exp(-V_{g,h+1}(s',c-g_h)/C'_{g,KL}),\nonumber\\
& Q_{r,h}(s,c,a)=r_h(s,a)+\sum_{s^{\prime}}P_h^*(s^{\prime}|s,a)V_{r,h+1}(s^{\prime},c-g_h), \nonumber \\& P_{r,h}^*\propto P_h^0(\cdot|s,a)\exp(-V_{r,h+1}(s^{\prime},c-g_h)/C'_{r,KL})
\end{align}
Once we obtain $Q$ values, we will find the policy as described in Algorithm~\ref{algo:rs-cmdp}, and update the $V$-value at step $h$. 
Then by Lemma 5 in \cite{Epigraph}, for any $C'_{j,KL}>0$ there exists $\rho>0$ such that the solution in (\ref{eq:worst_kl}) converges to $\rho$ showing the equivalence. We use this expression in (\ref{eq:equivalent}) for the robust value function update. 

For RNPG, and RCRPO we use the KL-divergence evaluator adapted to the episodic case exactly as described in \cite{Epigraph} as they are policy optimization based algorithms. 

\subsection{Preliminaries and Proof of Lemma~\ref{lem:total_cost}}
\label{app:prelim}

\paragraph{Normalization.}
The problem formulation allows $|r_h|,|g_h|\le1$. Replacing $r_h$ by $(r_h+1)/2$,
$g_h$ by $(g_h+1)/2$, and $b$ by $(b+H)/2$ maps every worst-case value by an
increasing affine map with slope $1/2$. It therefore leaves the feasible set of
every problem in this paper and the ordering of policies by reward unchanged,
and it rescales violation and sub-optimality by $1/2$. Throughout the appendix
we therefore assume $r_h,g_h\in[0,1]$ and $b\in[0,H]$.

\paragraph{Value ranges.}
The remaining budget satisfies $c_h=b-\sum_{h'<h}g_{h'}\in[b-H,b]\subseteq[-H,H]$.
Reward values lie in $[0,H]$. 

\paragraph{Operators act on the physical successor.}
For a function $V$ on $\cS\times\cC$ and a fixed budget $c'$, we write
$\cL_{\cQ_{h,s,a}}V(\cdot,c')=\min_{P\in\cQ_{h,s,a}}\sum_{s'\in\cS}P(s')V(s',c')$.
Every minimum over successor states below ranges over $s''\in\cS$ on the fixed
budget slice $c'$, never over the whole augmented space.

\paragraph{Grid model.}
The sample-complexity analysis (Appendices~\ref{proof:tv}--\ref{sec:continuous})
runs on the grid model $\hat M$ of Appendix~\ref{sec:continuous},
whose budget transition is $c\mapsto\Gamma(c-g_h(s,a))$. All definitions of
Section~\ref{sec:perstep} ($\cN_h$, $\cF_h$, $\cA_h$, $\Pi_{\mathrm{step}}$,
$V^{\diamond}$) are taken in $\hat M$, with $c-g_h(s,a)$ replaced by
$\Gamma(c-g_h(s,a))$. Under Assumption~\ref{assum:discrete}, $\Gamma$ is the
identity on reachable budgets, and $\hat M$ is the original problem.

\paragraph{Proof of Lemma~\ref{lem:total_cost}.}
Throughout, a history $\tau_h=(s_1,a_1,\dots,s_{h-1},a_{h-1},s_h)$ ends at the
state $s_h$ (before $a_h$ is chosen), and $c$ denotes the budget it determines.
Fix $\pi$, $h$, and a history $\tau_h$ ending at $(s,c)$. For a
history-dependent policy, the robust augmented utility value is the worst case,
over the continuation strategies of the dynamic adversary, of the expected
augmented utility collected from step $h$ on:
\[
V^{\pi}_{g,h}(\tau_h)=\min_{\xi\in\Xi}\E^{\pi,\xi}\Big[\sum_{h'=h}^{H+1}g^{\aug}_{h'}\ \Big|\ \tau_h\Big].
\]
For policies that are Markov on the augmented state, this coincides with the
robust Bellman recursion of Section~\ref{sec:thm1} by
rectangularity~\cite{iyengar2005robust}.

\emph{Step 1 (only the terminal term remains).} Since $g^{\aug}_{h'}\equiv0$ for
$h'\le H$ and $g^{\aug}_{H+1}(s_{H+1},c_{H+1})=-c_{H+1}$, the sum inside the
expectation equals $-c_{H+1}$ on every trajectory.

\emph{Step 2 (pathwise identity).} The budget evolves as
$c_{h'+1}=c_{h'}-g_{h'}(s_{h'},a_{h'})$. Summing from $h'=h$ to $H$ gives, on
every trajectory that extends $\tau_h$,
\[
-c_{H+1}=\sum_{h'=h}^{H}g_{h'}(s_{h'},a_{h'})-c_h .
\]

\emph{Step 3 (take the worst case).} Conditional on $\tau_h$, the budget
$c_h=c$ is a deterministic function of the history, so it can be taken outside
the expectation and the minimum:
\[
V^{\pi}_{g,h}(\tau_h)=\min_{\xi\in\Xi}\E^{\pi,\xi}\Big[\sum_{h'=h}^{H}g_{h'}(s_{h'},a_{h'})-c\ \Big|\ \tau_h\Big]
=\min_{\xi\in\Xi}\E^{\pi,\xi}\Big[\sum_{h'=h}^{H}g_{h'}(s_{h'},a_{h'})\ \Big|\ \tau_h\Big]-c .
\]

\emph{Consequences.} At $h=1$ the history is $s_1$ and $c=b$, so
$V^{\pi}_{g,1}(s_1,b)=\min_{\xi}\E^{\pi,\xi}[\sum_{h'=1}^{H}g_{h'}]-b$, which is
nonnegative if and only if~(\ref{eq:total_constraint}) holds. For general $h$,
$c=b-\sum_{h'<h}g_{h'}(s_{h'},a_{h'})$, where the sum is the utility already
incurred along $\tau_h$. Substituting,
\[
V^{\pi}_{g,h}(\tau_h)=\underbrace{\sum_{h'<h}g_{h'}(s_{h'},a_{h'})}_{\text{incurred}}
+\underbrace{\min_{\xi}\E^{\pi,\xi}\Big[\sum_{h'=h}^{H}g_{h'}\ \Big|\ \tau_h\Big]}_{\text{worst-case remaining}}-\,b ,
\]
which is nonnegative if and only if the incurred utility plus the worst-case
remaining utility is at least $b$. \hfill$\square$

\begin{fact}[Rectangular factorization]\label{fact:rect}
Under the dynamic adversary and $(s,a)$-rectangularity, for any
(history-dependent, randomized) policy $\pi$ and any history $\tau_h$ ending at
$(s,c)$, with $\alpha=\pi_h(\cdot\mid\tau_h)$,
\[
V^{\pi}_{j,h}(\tau_h)=\sum_{a}\alpha(a)\Big[(\mathrm{stage})_j
+\min_{P\in\cP_{h,s,a}}\sum_{s'}P(s')\,V^{\pi}_{j,h+1}\big(\tau_h a s'\big)\Big],
\]
where $(\mathrm{stage})_r=r_h(s,a)$ and $(\mathrm{stage})_g=0$. The adversary
observes $a$ and chooses independently on each branch. The reward and utility
minimizations are separate~\cite{iyengar2005robust}.
\end{fact}

\subsection{Proof of Lemma~\ref{lem:sub_optimality}: Markov Policies in the Original State Can Be Sub-optimal}
\label{app:suboptimality}

We prove Lemma~\ref{lem:sub_optimality} for a KL uncertainty set, and show that
the witness is per-step feasible, so that the separation also holds against
the benchmark $V^{\diamond}$ of our main theorems. The same argument applies
verbatim to any uncertainty set that preserves the nominal support (for
example, $\chi^2$).

\paragraph{Instance.}
Let $\cS=\{s_1,s_2,s_2',s_3\}$, $\cA=\{a_{\mathrm r},a_{\mathrm u}\}$, $H=3$ and threshold $b=1$.
All nominal transitions are action-independent:
$P^0_1(s_2\mid s_1)=P^0_1(s_2'\mid s_1)=\tfrac12$, and
$P^0_2(s_3\mid s_2)=P^0_2(s_3\mid s_2')=1$. Rewards and utilities are zero,
except $r(s_2,\cdot)=1$, $g(s_2',\cdot)=1$, $r(s_3,a_{\mathrm r})=1$ and
$g(s_3,a_{\mathrm u})=1$. (Figure~9 shows the nominal transitions.)
Consider the KL set of radius $\rho>0$ at every $(h,s,a)$. A KL ball contains
only distributions supported on the support of its center, so the step-2
transitions remain deterministic. At step~1, the ball is
$\{(p,1-p):\ \mathrm{KL}((p,1-p)\,\|\,(\tfrac12,\tfrac12))\le\rho\}
=\{(p,1-p):\ p\in[1-\bar p,\bar p]\}$, where $\bar p=\bar p(\rho)\in(\tfrac12,1]$
solves $\bar p\log(2\bar p)+(1-\bar p)\log(2(1-\bar p))=\rho$ (with $\bar p=1$
if $\rho\ge\log2$). Here $p$ denotes the probability of reaching $s_2$.

\paragraph{The augmented policy is per-step feasible and earns $1$.}
At $s_3$, play $a_{\mathrm u}$ if $c_3=1$ (the path went through $s_2$), and
play $a_{\mathrm r}$ if $c_3=0$ (the path went through $s_2'$). Along $s_1\to s_2\to s_3$, the total
utility is $1$ (from $a_{\mathrm u}$) and the total reward is $1$ (from $s_2$). Along
$s_1\to s_2'\to s_3$, the total utility is $1$ (from $s_2'$) and the total
reward is $1$ (from $a_{\mathrm r}$). These are the only reachable paths, so the threshold is
met on every path under every adversary. The policy is therefore per-step
feasible, and $V^{\diamond}_{r,1}\ge1$.

\paragraph{Every Markov policy in the original state earns at most $2(1-\bar p)$.}
Only the choice at $s_3$ matters. Let $q=\pi(a_{\mathrm u}\mid s_3)$. The total
utility is $\mathbbm{1}\{s_2'\}+\mathbbm{1}\{a_{\mathrm u}\}$, with worst-case expectation
$\min_p[(1-p)+q]=1-\bar p+q$. Feasibility requires $q\ge\bar p$. The total
reward is $\mathbbm{1}\{s_2\}+\mathbbm{1}\{a_{\mathrm r}\}$, with worst-case expectation
$\min_p[p+(1-q)]=1-\bar p+1-q\le2(1-\bar p)$. Since $\bar p>\tfrac12$,
\[
\max_{\pi\ \text{Markov in }s}V^{\pi}_{r,1}\ \le\ 2(1-\bar p)\ <\ 1\ \le\ V^{\diamond}_{r,1}\ \le\ V^{*}_{r,1}.
\]
The gap $2\bar p-1\approx\sqrt{2\rho}$ vanishes as $\rho\to0$, consistent with
the non-robust CMDP, where Markov policies are optimal. \hfill$\square$


\subsection{Static and Dynamic Adversaries Differ for RCMDP}
\label{app:static_dynamic}

We exhibit a policy, Markov on the augmented state, whose worst-case total
utility is strictly lower under the dynamic adversary than under the static
one.

\paragraph{Instance.}
Let $H=4$ and $b=\tfrac12$, and use a KL set of radius $\rho>0$ with $\bar p$
as in Appendix~\ref{app:suboptimality}. At $s_0$, action $a_L$ moves to $s_L$
and action $a_R$ moves to $s_R$, both deterministically. Both $s_L$ and
$s_R$ move deterministically to $s_M$. At $s_M$, the single action moves to
$s_X$ or $s_Y$ with nominal probability $\tfrac12$ each; this is the only
uncertain transition, and we write $p$ for the probability of $s_X$, with
$p\in[1-\bar p,\bar p]$. At step~4, the states $s_X$ and $s_Y$ each have a
reward action ($r=1$, $g=0$) and a utility action ($r=0$, $g=\tfrac12$). All
other rewards and utilities are zero, except $g(s_R,\cdot)=\tfrac12$. Hence the
budget at step~4 is $c_4=\tfrac12$ after $s_L$ and $c_4=0$ after $s_R$.

\paragraph{Policy.}
At $s_0$, play $a_L$ and $a_R$ with probability $\tfrac12$ each. At step~4,
play the utility action at $(s_X,\tfrac12)$ and at $(s_Y,0)$, and the reward
action at $(s_Y,\tfrac12)$ and at $(s_X,0)$.

\paragraph{Evaluation.}
The total utility is $p/2$ after $s_L$ and $\tfrac12+(1-p)/2$ after $s_R$.
\begin{itemize}
\item \emph{Static adversary.} A single $p$ is used on both branches, and the
expected total utility is $\tfrac12\cdot\tfrac p2+\tfrac12\big(\tfrac12+\tfrac{1-p}2\big)=\tfrac12$
for every $p$. The policy meets the constraint with equality.
\item \emph{Dynamic adversary.} The adversary observes the budget at $s_M$ and
chooses $p=1-\bar p$ after $s_L$ and $p=\bar p$ after $s_R$. The expected total
utility is $\tfrac12\cdot\tfrac{1-\bar p}{2}+\tfrac12\big(\tfrac12+\tfrac{1-\bar p}{2}\big)
=\tfrac12-\tfrac{\bar p-\frac12}{2}<\tfrac12$. The policy violates the
constraint.
\end{itemize}
The two readings coincide for policies that are Markov in the original state.
Here the state $s_M$ is reached with two different budgets, and the static
adversary must use the same kernel for both. \hfill$\square$

\subsection{An Illustrative Instance}
\label{app:example}

Consider $H=2$, $b=\tfrac12$, $g_1\equiv0$, $r_1\equiv0$, and singleton
uncertainty sets. At $s_0$, action $x$ moves to $s_A$ or $s_B$ with
probability $\tfrac12$ each, and action $y$ moves to $s_C$. At step~2:
\begin{itemize}
\item at $s_A$, action $a_1$ yields $(g,r)=(1,1)$ and action $a_2$ yields
$(g,r)=(\tfrac12,1-\delta)$, for some $\delta\in(0,1)$;
\item at $s_B$, the only action yields $(0,0)$;
\item at $s_C$, the only action yields $(1,0)$.
\end{itemize}
With $c=\tfrac12$ at step~2, $Q_{g,2}=g-\tfrac12$. Hence
$Q_{g,2}(s_A,\tfrac12,a_1)=\tfrac12$, $Q_{g,2}(s_A,\tfrac12,a_2)=0$,
$V_{g,2}(s_B,\tfrac12)=-\tfrac12$ and $V_{g,2}(s_C,\tfrac12)=\tfrac12$.

\paragraph{The per-step benchmark.}
$(s_B,\tfrac12)\notin\cF_2$, so $x\notin\cA_1(s_0,\tfrac12)$, and the only
per-step feasible choice is $y$. Therefore $V^{\diamond}_{r,1}=0$.

\paragraph{The time-consistent recursion.}
At $(s_A,\tfrac12)$, the recursion picks $a_1$. At the root,
$Q^{\star}_{g,1}(x)=\tfrac12\cdot\tfrac12+\tfrac12\cdot(-\tfrac12)=0$, so $x$ is
feasible and earns reward $\tfrac12$. Thus
$V^{\pi^{\star}}_{r,1}=\tfrac12>V^{\diamond}_{r,1}=0$, which illustrates
Remark~\ref{rem:why_lp}(ii). For a KL set of radius $\rho>0$ with the same
supports, $Q^{\star}_{g,1}(x)<0$, and the LP mixes $x$ with $y$. The recursion
still earns a strictly positive reward, which approaches $\tfrac12$ as
$\rho\to0$.

\begin{remark}[Why Assumption~\ref{ass:feas} asks for pathwise feasibility]
\label{rem:F}
A weaker assumption would be that \emph{some} policy satisfies the expected
constraint, i.e., $b\le G_1(s_1)$, where $G_h(s)$ is the largest worst-case
remaining utility from $s$ at step $h$. Under this weaker assumption, even the
exact recursion $\pi^{\star}$ can violate the constraint. Modify the instance
above: remove $y$, so $x$ is the only action at $s_0$, and give $s_A$ two
actions, $(g,r)=(1,0)$ and $(g,r)=(\tfrac12,1)$. Then $G_1(s_0)=\tfrac12=b$, and
playing $(1,0)$ at $s_A$ meets the constraint with equality. The recursion,
however, maximizes reward at $(s_A,\tfrac12)$ subject to its local constraint
and picks $(\tfrac12,1)$, which leaves $V^{\star}_{g,2}(s_A,\tfrac12)=0$. Since
$V_{g,2}(s_B,\tfrac12)=-\tfrac12$, we get $Q^{\star}_{g,1}(s_0,\tfrac12,x)=-\tfrac14$.
The root LP is infeasible, and $\pi^{\star}$ violates the constraint by
$\tfrac14$. The recursion spends at $s_A$ the slack that the root needed to
cover $s_B$. Pathwise feasibility rules this out: every successor must meet
the threshold on its own, so no state relies on slack left elsewhere.
Accordingly, $(s_0,\tfrac12)\notin\cF_1$ here, because $s_B$ is reachable.
\end{remark}

\subsection{Proof of Theorem~\ref{thm:tv} (TV Uncertainty Set)}
\label{proof:tv}

\subsubsection{One-step concentration}

\begin{lemma}\label{lem:diff_v}
Fix $(h,s,a)$ and a function $W:\cS\to[0,V_{\max}]$ that is independent of the
samples drawn at $(h,s,a)$. For any $\theta\in(0,1)$ and $\delta\in(0,1)$, with
probability at least $1-\delta$,
\[
\big|\cL_{\cP_{h,s,a}}W-\cL_{\hat\cP_{h,s,a}}W\big|
\ \le\ V_{\max}\sqrt{\frac{\log\big(6V_{\max}/(\theta\delta)\big)}{2N}}+2\theta .
\]
\end{lemma}
\begin{proof}
By Proposition~1 of~\cite{panaganti2022sample}, with
$D_{\mathrm{TV}}=\tfrac12\|\cdot\|_1$, for any center
$Q\in\{P^0_h(\cdot\mid s,a),\hat P^0_h(\cdot\mid s,a)\}$ of a TV ball of radius
$\rho$,
\[
\cL W=-\inf_{\eta\in[0,2V_{\max}/\rho]}\Big\{\E_{s'\sim Q}\big[(\eta-W(s'))_+\big]
+\rho\big(\eta-\min_{s''\in\cS}W(s'')\big)_+-\eta\Big\}.
\]
For $\eta\ge\max W$ the objective equals $\rho(\eta-\min W)-\E_QW$, which is
nondecreasing in $\eta$, so the infimum is attained in $[0,V_{\max}]$; the range
$[0,2V_{\max}/\rho]$ therefore loses nothing for $\rho\le2$.
(As a check, if $W$ is $0$ at a state outside $\mathrm{supp}(Q)$ and $1$ on
$\mathrm{supp}(Q)$, the formula gives $1-\rho$, the value obtained by moving
mass $\rho$ onto the minimizer.) Only the first term depends on $Q$. Using
$|\inf f-\inf g|\le\sup|f-g|$,
\[
\big|\cL_{\cP}W-\cL_{\hat\cP}W\big|
\le\sup_{\eta\in[0,2V_{\max}/\rho]}\big|(\hat P^0_h-P^0_h)(\eta-W)_+\big|.
\]
For $\eta\ge V_{\max}$, $(\eta-W)_+=\eta-W$, so the term equals
$|(\hat P^0_h-P^0_h)W|$. For $\eta\in[0,V_{\max}]$, Lemma~\ref{lem:approx-aug}
reduces the supremum to a maximum over the net $\cN_W(\theta)$ plus $2\theta$.
All functions involved take values in $[0,V_{\max}]$, and there are at most
$2V_{\max}/\theta+1\le3V_{\max}/\theta$ of them. Hoeffding's inequality
(Lemma~\ref{lem:hoeffding}), with a union bound, gives the claim.
\end{proof}

\begin{lemma}[Covering]\label{lem:cover-aug}
For $W:\cS\to[0,V_{\max}]$, let
$\cU_W=\{(\eta-W)_+:\eta\in[0,V_{\max}]\}$ and
$\cN_W(\theta)=\{(\min(k\theta,V_{\max})-W)_+:k=1,\dots,\lceil V_{\max}/\theta\rceil\}$.
Then $\cN_W(\theta)$ is a $\theta$-cover of $\cU_W$ in $\|\cdot\|_\infty$,
$|\cN_W(\theta)|\le2V_{\max}/\theta$, and every $\nu\in\cN_W(\theta)$ takes
values in $[0,V_{\max}]$.
\end{lemma}
\begin{proof}
$|(\eta-W)_+-(\eta'-W)_+|\le|\eta-\eta'|$ pointwise, and every
$\eta\in[0,V_{\max}]$ is within $\theta$ of a grid point. The range claim holds
because $W\ge0$.
\end{proof}

\begin{lemma}[Approximation by the cover]\label{lem:approx-aug}
$\sup_{\mu\in\cU_W}|(\hat P^0_h-P^0_h)\mu|\le\max_{\nu\in\cN_W(\theta)}|(\hat P^0_h-P^0_h)\nu|+2\theta$.
\end{lemma}
\begin{proof}
For $\|\mu-\nu\|_\infty\le\theta$, the triangle inequality gives
$|(\hat P-P)\mu|\le|(\hat P-P)\nu|+\|\hat P\|_1\theta+\|P\|_1\theta$.
\end{proof}

\begin{lemma}[1-Lipschitz]\label{lem:L-1lip}
For any set of kernels $\cQ$ and bounded $W_1,W_2$,
$|\cL_{\cQ}W_1-\cL_{\cQ}W_2|\le\|W_1-W_2\|_\infty$.
\end{lemma}
\begin{proof}
For every $P\in\cQ$, $PW_1\le PW_2+\|W_1-W_2\|_\infty$. Take the infimum over
$P$ and exchange the roles of $W_1$ and $W_2$.
\end{proof}

\subsubsection{The good event and value propagation}

Let $\pi^{\diamond}$ be the fixed deterministic Markov policy of
Proposition~\ref{prop:perstep_dp}, extended arbitrarily outside $\cF$. Let
$\hat V^{\pi}_{j,h}$ denote the robust values of $\pi$ under the empirical
uncertainty set $\hat\cP$. By Line~12 of Algorithm~\ref{algo:rs-cmdp},
$\hat V_{j,h}=\hat V^{\hat\pi}_{j,h}$. Define
\[
\varepsilon_{N,\theta}:=V_{\max}\sqrt{\frac{\log\big(48H^2|\cS||\cA||\cC|/(\theta\delta)\big)}{2N}}+2\theta ,
\]
and let $\cE$ be the event that, for all $h\in[H]$, $(s,a)$, $c'\in\cC$ and
\[
W\in\big\{\hat V^{\hat\pi}_{r,h+1},\ \hat V^{\hat\pi}_{g,h+1}+H,\ \hat V^{\pi^{\diamond}}_{r,h+1},\ \hat V^{\pi^{\diamond}}_{g,h+1}+H\big\}(\cdot,c'),
\]
we have $|\cL_{\cP_{h,s,a}}W-\cL_{\hat\cP_{h,s,a}}W|\le\varepsilon_{N,\theta}$.

\begin{lemma}\label{lem:event}
$\Pr(\cE)\ge1-\delta$.
\end{lemma}
\begin{proof}
Algorithm~\ref{algo:rs-cmdp} draws fresh samples at every step. Each $W$ above
is a function of the samples drawn at steps $h+1,\dots,H$ only: the policy
$\hat\pi_{h+1:H}$ depends only on those samples, and $\pi^{\diamond}$ is fixed.
Conditioning on those samples, $W$ is fixed and independent of the step-$h$
samples, so Lemma~\ref{lem:diff_v} applies with confidence
$\delta/(4H|\cS||\cA||\cC|)$. A union bound over the $4H|\cS||\cA||\cC|$ choices
of $(h,s,a,c')$ and $W$ gives the claim, using $6V_{\max}\cdot4H=48H^2$.
\end{proof}

\begin{lemma}[Value propagation]\label{lem:V-propagation}
On $\cE$, for $\pi\in\{\hat\pi,\pi^{\diamond}\}$, $j\in\{r,g\}$, $h\in[H+1]$ and
$(s,\hat c)\in\cS\times\cC$,
$|V^{\pi}_{j,h}(s,\hat c)-\hat V^{\pi}_{j,h}(s,\hat c)|\le(H-h+1)\,\varepsilon_{N,\theta}$.
\end{lemma}
\begin{proof}
Backward induction; the values coincide at $h=H+1$. Write
$c'=\Gamma(\hat c-g_h(s,a))$. Then
\[
Q^{\pi}_{j,h}-\hat Q^{\pi}_{j,h}
=\underbrace{\cL_{\cP}V^{\pi}_{j,h+1}(\cdot,c')-\cL_{\cP}\hat V^{\pi}_{j,h+1}(\cdot,c')}_{(A)}
+\underbrace{\cL_{\cP}\hat V^{\pi}_{j,h+1}(\cdot,c')-\cL_{\hat\cP}\hat V^{\pi}_{j,h+1}(\cdot,c')}_{(B)}.
\]
By Lemma~\ref{lem:L-1lip} and the induction hypothesis,
$|(A)|\le(H-h)\varepsilon_{N,\theta}$. On $\cE$, using translation equivariance
for $j=g$, $|(B)|\le\varepsilon_{N,\theta}$. Averaging over
$\pi_h(\cdot\mid s,\hat c)$ gives the bound for $V$. Only the empirical values
$\hat V^{\pi}_{h+1}$ are concentrated; the true values enter only
through~$(A)$.
\end{proof}

\subsubsection{Joint induction on the empirical model}

\begin{fact}[Empirical supports]\label{ass:support}
$\mathrm{supp}(P)\subseteq\cN_h(s,a)$ for every $P\in\hat\cP_{h,s,a}$. For TV this
is trivial ($\cN_h=\cS$). For KL and $\chi^2$, every $P$ in a ball around
$\hat P^0_h$ is supported on $\mathrm{supp}(\hat P^0_h)$, and
$\mathrm{supp}(\hat P^0_h)\subseteq\mathrm{supp}(P^0_h)=\cN_h(s,a)$ because every
sample is drawn from $P^0_h$. This is not an additional assumption.
\end{fact}

\begin{lemma}[Joint induction]\label{lem:joint}
For every $h\in[H]$ and
$(s,\hat c)\in\cF_h$:
\begin{enumerate}
\item[(a)] $\hat Q^{\hat\pi}_{g,h}(s,\hat c,a)\ge-\kappa$ for all
$a\in\cA_h(s,\hat c)$; in particular, the LP~(\ref{eq:lp}) is feasible and
$\pi^{\diamond}_h(\cdot\mid s,\hat c)$ is a feasible point of it;
\item[(b)] $\hat V^{\hat\pi}_{g,h}(s,\hat c)\ge-\kappa$;
\item[(c)] $\hat V^{\hat\pi}_{r,h}(s,\hat c)\ge\hat V^{\pi^{\diamond}}_{r,h}(s,\hat c)$.
\end{enumerate}
\end{lemma}
\begin{proof}
Backward induction from $h=H+1$, where $\hat V_{g,H+1}=-\hat c\ge0$ on
$\cF_{H+1}$ and both reward values vanish.
(a) For $a\in\cA_h(s,\hat c)$, every $P\in\hat\cP_{h,s,a}$ is supported on
$\cN_h(s,a)$ (Fact~\ref{ass:support}), and all those successors lie in
$\cF_{h+1}$. There, $\hat V^{\hat\pi}_{g,h+1}\ge-\kappa$ by (b) at $h+1$ (or
$\ge0$ if $h=H$). Hence $\hat Q^{\hat\pi}_{g,h}(s,\hat c,a)\ge-\kappa$.
(b) By (a), the LP is feasible at $(s,\hat c)$, and its constraint gives the
claim.
(c) $\pi^{\diamond}_h(\cdot\mid s,\hat c)$ is supported on $\cA_h(s,\hat c)$ and
is feasible for the LP by (a), so
$\hat V^{\hat\pi}_{r,h}\ge\langle\pi^{\diamond}_h,\hat Q^{\hat\pi}_{r,h}\rangle$.
For $a\in\cA_h(s,\hat c)$, all empirical successors lie in $\cF_{h+1}$, where
(c) at $h+1$ gives $\hat V^{\hat\pi}_{r,h+1}\ge\hat V^{\pi^{\diamond}}_{r,h+1}$.
Monotonicity of $\cL$ yields
$\hat Q^{\hat\pi}_{r,h}(s,\hat c,a)\ge\hat Q^{\pi^{\diamond}}_{r,h}(s,\hat c,a)$,
and hence
$\langle\pi^{\diamond}_h,\hat Q^{\hat\pi}_{r,h}\rangle\ge\hat V^{\pi^{\diamond}}_{r,h}$.
\end{proof}

\subsubsection{Completing the proof}

We work on the grid model; Appendix~\ref{sec:continuous} transfers the result
to continuous budgets. Let $\hat c_1=\Gamma(b)$. By Assumption~\ref{ass:feas}
and Lemma~\ref{lem:disc-bench}, $(s_1,\hat c_1)\in\cF_1$. We work on $\cE$,
which has probability at least $1-\delta$ by Lemma~\ref{lem:event}.

\paragraph{Violation.}
By Lemma~\ref{lem:joint}(b) and Lemma~\ref{lem:V-propagation},
\[
V^{\hat\pi}_{g,1}(s_1,\hat c_1)\ \ge\ \hat V^{\hat\pi}_{g,1}(s_1,\hat c_1)-H\varepsilon_{N,\theta}\ \ge\ -\kappa-H\varepsilon_{N,\theta}.
\]
By Lemma~\ref{lem:total_cost}, this bounds the violation of the robust
expected total-utility constraint.

\paragraph{Reward.}
By Lemmas~\ref{lem:V-propagation} and~\ref{lem:joint}(c),
\[
V^{\hat\pi}_{r,1}\ \ge\ \hat V^{\hat\pi}_{r,1}-H\varepsilon_{N,\theta}
\ \ge\ \hat V^{\pi^{\diamond}}_{r,1}-H\varepsilon_{N,\theta}
\ \ge\ V^{\diamond}_{r,1}-2H\varepsilon_{N,\theta},
\]
all evaluated at $(s_1,\hat c_1)$.

\paragraph{Sample size.}
Take $\theta=\epsilon/(16H)$, so that $2H\theta=\epsilon/8$. With $V_{\max}=2H$,
the condition $H\varepsilon_{N,\theta}\le\epsilon/2$ reduces to
$2H^2\sqrt{L/(2N)}\le3\epsilon/8$, where
$L=\log(768H^3|\cS||\cA||\cC|/(\epsilon\delta))$. It holds as soon as
\[
N\ \ge\ \frac{15H^4}{\epsilon^2}\,\log\!\Big(\frac{768H^3|\cS||\cA||\cC|}{\epsilon\delta}\Big).
\]
With $\kappa=\epsilon/2$, we obtain $\mathrm{Violation}(\hat\pi)\le\epsilon$ and
$V^{\hat\pi}_{r,1}\ge V^{\diamond}_{r,1}-\epsilon$ with probability at least
$1-\delta$, using $N_{\mathrm{tot}}=N|\cS||\cA|H$ samples. \hfill$\square$

\subsection{Proof of Theorem~\ref{thm:chi} ($\chi^2$ Uncertainty Set)}
\label{proof:chi}

The proof of Theorem~\ref{thm:chi} is identical to that of
Theorem~\ref{thm:tv}, with Lemma~\ref{lem:diff_v} replaced by
Lemma~\ref{lem:chi_dev}. Lemmas~\ref{lem:event}--\ref{lem:joint} use only the
one-step deviation bound, Lemma~\ref{lem:L-1lip}, and the support condition,
all of which hold here. Let
\[
C_\rho=\sqrt{1+\rho},\qquad \kappa_\rho=\frac{C_\rho}{C_\rho-1},\qquad \bar\eta=\kappa_\rho V_{\max},
\]
and write $\|Y\|_{2,Q}=(\E_Q[Y^2])^{1/2}$.

\begin{lemma}[Dual]\label{lem:chi2-dual}
For $W:\cS\to[0,V_{\max}]$ and a center $Q$ of a $\chi^2$ ball of radius $\rho$,
with $D_\chi(P,Q)=\sum_{s'}(P(s')-Q(s'))^2/Q(s')$,
\[
\cL W=-\inf_{\eta\in[0,\bar\eta]}\Big\{C_\rho\,\big\|(\eta-W)_+\big\|_{2,Q}-\eta\Big\}.
\]
\end{lemma}
\begin{proof}
Cressie--Read duality with exponent $2$ gives
$\inf_{P:\,D_\chi(P,Q)\le\rho}\E_PW=\sup_{\eta\in\R}\{\eta-C_\rho\|(\eta-W)_+\|_{2,Q}\}$
(see, e.g.,~\cite{xu2023improved}). For $\eta<0$ the objective equals
$\eta<0$. At $\eta=\min_{s}W(s)$ it equals $\min_sW(s)\ge0$. For
$\eta>\bar\eta\ge V_{\max}$, pointwise $(\eta-W)_+\ge\eta-V_{\max}$, so the
objective is at most $\eta-C_\rho(\eta-V_{\max})=C_\rho V_{\max}-(C_\rho-1)\eta<0$.
The supremum is therefore attained in $[0,\bar\eta]$.
\end{proof}

\begin{lemma}[Covering]\label{lem:chi2-cover-aug}
Let $\cN_\rho(\theta)\subset[0,\bar\eta]$ be a $\theta$-net of $[0,\bar\eta]$ with
$|\cN_\rho(\theta)|\le\bar\eta/\theta+1$. For any two kernels $P,\hat P$,
\[
\sup_{\eta\in[0,\bar\eta]}\Big|\,\|(\eta-W)_+\|_{2,P}-\|(\eta-W)_+\|_{2,\hat P}\Big|
\le\max_{\nu\in\cN_\rho(\theta)}\Big|\,\|(\nu-W)_+\|_{2,P}-\|(\nu-W)_+\|_{2,\hat P}\Big|+2\theta .
\]
\end{lemma}
\begin{proof}
$\eta\mapsto(\eta-W)_+$ is 1-Lipschitz in $\|\cdot\|_\infty$, and
$\|\cdot\|_{2,Q}\le\|\cdot\|_\infty$. Apply the reverse triangle inequality
under each of $P$ and $\hat P$.
\end{proof}

\begin{lemma}[Concentration of an $L^2$ norm]\label{lem:chi2-conc-aug}
Let $Y:\cS\to[0,B]$ be fixed and independent of the samples, and let $\hat P$
be the empirical distribution of $N$ i.i.d.\ samples from $P$. With probability
at least $1-\delta$,
\[
\big|\,\|Y\|_{2,\hat P}-\|Y\|_{2,P}\big|\ \le\ 3B\sqrt{\frac{\log(2/\delta)}{N}} .
\]
\end{lemma}
\begin{proof}
Let $L=\log(2/\delta)$, $a=\E_PY^2$ and $\hat a=\E_{\hat P}Y^2$. The variables
$Y^2(s'_i)\in[0,B^2]$ have variance at most $\E_PY^4\le B^2a$. Bernstein's
inequality (Lemma~\ref{lem:bernstein}) gives
$|\hat a-a|\le B\sqrt{2aL/N}+2B^2L/(3N)$.
If $a\ge B^2L/N$, then
$|\sqrt{\hat a}-\sqrt a|\le|\hat a-a|/\sqrt a\le B\sqrt{2L/N}+\tfrac23B\sqrt{L/N}\le3B\sqrt{L/N}$.
Otherwise, $a<B^2L/N$ and $\hat a\le a+|\hat a-a|<(1+\sqrt2+\tfrac23)B^2L/N$. Both
square roots then lie in $[0,1.8B\sqrt{L/N}]$, so they differ by at most
$1.8B\sqrt{L/N}$.
\end{proof}

\begin{lemma}[One-step deviation]\label{lem:chi2-operator-aug}\label{lem:chi_dev}
For fixed $W:\cS\to[0,V_{\max}]$ independent of the step-$h$ samples and any
$\theta\in(0,1)$, with probability at least $1-\delta$,
\[
\big|\cL_{\cP_{h,s,a}}W-\cL_{\hat\cP_{h,s,a}}W\big|
\ \le\ C_\rho\Big(3\kappa_\rho V_{\max}\sqrt{\frac{\log\big(2(\bar\eta/\theta+1)/\delta\big)}{N}}+2\theta\Big).
\]
\end{lemma}
\begin{proof}
By Lemma~\ref{lem:chi2-dual} and $|\inf f-\inf g|\le\sup|f-g|$, the deviation is
at most
$C_\rho\sup_{\eta\in[0,\bar\eta]}\big|\|(\eta-W)_+\|_{2,\hat P^0_h}-\|(\eta-W)_+\|_{2,P^0_h}\big|$.
Apply Lemma~\ref{lem:chi2-cover-aug}, then Lemma~\ref{lem:chi2-conc-aug} with
$Y=(\nu-W)_+\in[0,\bar\eta]$ (so $B=\bar\eta=\kappa_\rho V_{\max}$) and a union
bound over the net.
\end{proof}

\paragraph{Proof of Theorem~\ref{thm:chi}.}
Repeat the proof of Theorem~\ref{thm:tv}, applying Lemma~\ref{lem:chi_dev}
with confidence $\delta/(4H|\cS||\cA||\cC|)$. With $V_{\max}=2H$ and
$\theta=\epsilon/(8C_\rho H)$, we have $2HC_\rho\theta=\epsilon/4$, and
$H\varepsilon_N\le\epsilon/2$ holds as soon as
\[
N\ \ge\ \frac{576\,(C_\rho\kappa_\rho)^2H^4}{\epsilon^2}\,L,\qquad
L=\log\!\Big(\frac{8H|\cS||\cA||\cC|\,(16C_\rho\kappa_\rho H^2/\epsilon+1)}{\delta}\Big).
\]
Since $(C_\rho\kappa_\rho)^2=(1+\rho)^2/(\sqrt{1+\rho}-1)^2=c(\rho)$, this gives
$N_{\mathrm{tot}}=N|\cS||\cA|H$ of the order stated in Theorem~\ref{thm:chi}.
\hfill$\square$

\subsection{Proof of Theorem~\ref{thm:kl} (KL Uncertainty Set)}
\label{proof:kl}

For KL, the one-step deviation bound holds \emph{uniformly} over all value
functions, so the fresh-sample argument is not needed. Let
$\zeta=\min_{h,s,a,s':P^0_h(s'\mid s,a)>0}P^0_h(s'\mid s,a)$.

\begin{lemma}[One-step deviation, uniform in $W$]\label{lem:KL-aug-support}\label{lem:kl_dev}
Fix $(h,s,a)$ and suppose $N\ge2\log(2|\cS|/\delta)/\zeta^2$. With probability at
least $1-\delta$, simultaneously for all $W:\cS\to[0,V_{\max}]$,
\[
\big|\cL_{\cP_{h,s,a}}W-\cL_{\hat\cP_{h,s,a}}W\big|
\ \le\ \frac{2V_{\max}}{\rho}\sqrt{\frac{\log(2|\cS|/\delta)}{2N\zeta^2}} .
\]
\end{lemma}
\begin{proof}
By KL duality (see, e.g.,~\cite{xu2023improved}), for a center $Q$,
$\cL W=-\inf_{\lambda\in[0,V_{\max}/\rho]}f_Q(\lambda)$ with
$f_Q(\lambda)=\lambda\rho+\lambda\log\E_Q[e^{-W/\lambda}]$ for $\lambda>0$,
extended continuously at $\lambda=0$ by $f_Q(0)=-\min_{s'\in\mathrm{supp}(Q)}W(s')$.

\emph{Ratio event.} For $s'\in\mathrm{supp}(P^0)$, apply Hoeffding's inequality to
$\mathbb{1}\{s'_i=s'\}/P^0(s')\in[0,1/\zeta]$, which has mean $1$. A union bound
over at most $|\cS|$ successors gives, with probability at least $1-\delta$,
$\max_{s'}|\hat P^0(s')/P^0(s')-1|\le t:=\sqrt{\log(2|\cS|/\delta)/(2N\zeta^2)}\le\tfrac12$.
In particular, $\hat P^0$ charges every state of $\mathrm{supp}(P^0)$; since
always $\mathrm{supp}(\hat P^0)\subseteq\mathrm{supp}(P^0)$, the two supports
coincide.

\emph{Uniform comparison.} For any positive weights $w$,
$\E_{\hat P}[w]/\E_{P}[w]\in[1-t,1+t]$, because it is a ratio of sums over
$\mathrm{supp}(P^0)$ whose termwise ratios lie in $[1-t,1+t]$. With
$w=e^{-W/\lambda}$ and $|\log(1+x)|\le2|x|$ for $|x|\le\frac12$, for every
$\lambda\in(0,V_{\max}/\rho]$,
$|f_{\hat P}(\lambda)-f_P(\lambda)|=\lambda|\log(\E_{\hat P}w/\E_Pw)|\le\frac{V_{\max}}{\rho}\cdot2t$.
At $\lambda=0$, the two functions coincide because the supports coincide. The
claim follows from $|\inf f-\inf g|\le\sup|f-g|$. The event does not depend on
$W$.
\end{proof}

\paragraph{Proof of Theorem~\ref{thm:kl}.}
Repeat the proof of Theorem~\ref{thm:tv}, with Lemma~\ref{lem:kl_dev} applied at
confidence $\delta/(|\cS||\cA|H)$. The union bound runs over $(h,s,a)$ only,
and
$\varepsilon_N=\frac{4H}{\rho\zeta}\sqrt{\log(2|\cS|^2|\cA|H/\delta)/(2N)}$.
The condition $H\varepsilon_N\le\epsilon/2$ gives
$N=\cO\big(H^4\log(|\cS|^2|\cA|H/\delta)/(\rho^2\zeta^2\epsilon^2)\big)$, and hence
$N_{\mathrm{tot}}=N|\cS||\cA|H$. It also forces $t\le\epsilon\rho/(8H^2)$,
so the condition $t\le\frac12$ of Lemma~\ref{lem:kl_dev} holds whenever
$\epsilon\le4H^2/\rho$. \hfill$\square$

\subsection{Continuous Utilities (Relaxing Assumption~\ref{assum:discrete})}
\label{sec:continuous}

Let $\cC$ be the grid on $[-H,H]$ with resolution $\epsilon_0$, containing
$-H$, and let
\[
\Gamma(x)=\max\{\hat c\in\cC:\ \hat c\le\max(x,-H)\}
\]
round down, clamping at $-H$. Let $M$ denote the original problem with
continuous budget $c_h$, and let $\hat M$ be the grid problem with tracked budget
$\hat c_1=\Gamma(b)$ and $\hat c_{h+1}=\Gamma(\hat c_h-g_h)$. Both share the
physical kernels, the rewards, and the uncertainty sets. By induction,
\begin{equation}\label{eq:grid}
0\ \le\ c_h-\hat c_h\ \le\ h\,\epsilon_0\qquad\text{for all }h\in[H+1].
\end{equation}
Indeed, without clamping each step lowers $\hat c$ by at most $\epsilon_0$ more
than $c$. With clamping, $\hat c_{h+1}=-H\le b-H\le c_{h+1}$, and
$c_{h+1}+H<c_h-\hat c_h$ because $\hat c_h-g_h<-H$.

\begin{lemma}[Violation transfer]\label{lem:violation_discrete_to_original}
For any policy $\pi$ on $\hat M$, executed in $M$ by tracking $\hat c_h$,
\[
\min_{\xi}\E^{\pi,\xi}_{M}[-c_{H+1}]\ \ge\ \min_{\xi}\E^{\pi,\xi}_{\hat M}[-\hat c_{H+1}]-(H+1)\epsilon_0,
\]
and the worst-case reward values of $\pi$ in $M$ and $\hat M$ coincide.
\end{lemma}
\begin{proof}
Both executions generate the same trajectories of $(s_h,a_h)$, and hence of
$(c_h,\hat c_h)$, under every adversary strategy, because $\hat c_h$ is a
function of the history. Apply~\eqref{eq:grid} pathwise, then take expectations
and the minimum over $\xi$.
\end{proof}

\begin{lemma}[The benchmark is preserved]\label{lem:disc-bench}
$V^{\diamond}_{r,1}(\hat M)\ge V^{\diamond}_{r,1}(M)$. In particular,
Assumption~\ref{ass:feas} for $M$ implies it for $\hat M$.
\end{lemma}
\begin{proof}
Let $\pi$ be per-step feasible in $M$. Define $\pi'$ on $\hat M$ to play
$\pi(\cdot\mid\text{history})$, reconstructing $c_h$ from the history. This
policy is history-dependent on $\hat M$, which the per-step class allows.
Pathwise, $\hat c_{H+1}\le c_{H+1}\le0$ by~\eqref{eq:grid}, so $\pi'$ is
per-step feasible in $\hat M$ by Lemma~\ref{lem:perstep_char} (applied in
$\hat M$, where the terminal condition reads $\hat c_{H+1}\le0$). Its rewards
equal those of $\pi$ path by path, so its worst-case reward is the same. Taking
the supremum, and noting that Proposition~\ref{prop:perstep_dp} in $\hat M$
covers history-dependent policies, gives the claim.
\end{proof}

\paragraph{Conclusion.}
Run Algorithm~\ref{algo:rs-cmdp} on $\hat M$ with $\epsilon_0=\epsilon/(3(H+1))$,
$\kappa=\epsilon/3$, and $N$ chosen so that $H\varepsilon_N\le\epsilon/3$. By
the proof of Theorem~\ref{thm:tv} on $\hat M$ and
Lemma~\ref{lem:violation_discrete_to_original}, the violation in $M$ is at most
$\epsilon/3+\epsilon/3+\epsilon/3=\epsilon$. By Lemma~\ref{lem:disc-bench},
$V^{\hat\pi}_{r,1}(M)=V^{\hat\pi}_{r,1}(\hat M)\ge V^{\diamond}_{r,1}(\hat M)-\tfrac{2\epsilon}{3}\ge V^{\diamond}_{r,1}(M)-\epsilon$.
The grid has $|\cC|=\cO(H^2/\epsilon)$ points, which enter only
logarithmically.

\subsection{Supporting Results}

\begin{lemma}[Hoeffding]\label{lem:hoeffding}
If $X_1,\dots,X_n$ are i.i.d.\ with mean $\mu$ and $X_i\in[a,b]$, then for all
$t>0$,
$\Pr\big(|\frac1n\sum_iX_i-\mu|\ge t\big)\le2\exp\big(-2nt^2/(b-a)^2\big)$.
\end{lemma}

\begin{lemma}[Bernstein]\label{lem:bernstein}
If $X_1,\dots,X_n$ are i.i.d.\ with mean $\mu$, variance at most $\sigma^2$, and
$X_i\in[0,M]$, then with probability at least $1-\delta$,
$|\frac1n\sum_iX_i-\mu|\le\sqrt{2\sigma^2\log(2/\delta)/n}+2M\log(2/\delta)/(3n)$.
\end{lemma}

\end{document}